\documentclass{article} 
\usepackage{amsmath,amssymb,amsthm}
\usepackage{float}

\usepackage[final]{graphicx} 

\numberwithin{equation}{section}
\theoremstyle{plain}
\newtheorem{theo}{Theorem}[section]
\newtheorem{prop}[theo]{Proposition}
\newtheorem{lemm}[theo]{Lemma}
\newtheorem{coro}[theo]{Corollary}
\theoremstyle{definition}
\newtheorem{defi}[theo]{Definition}
\newtheorem{exam}[theo]{Example}
\newtheorem{obse}[theo]{Observation}
\newtheorem{rema}[theo]{Remark}
\newtheorem{conv}[theo]{Convention}

\begin{document}
\title{Separation and collapse of equilibria inequalities on AND--OR trees without shape constraints}
\author{ 
Fuki Ito\thanks{Partially supported by JST, the establishment of university fellowships towards the creation of science technology innovation, Grant Number JPMJFS2139. The current affiliation is Serverworks Co., Ltd. E-mail {\tt fukiito2022@gmail.com}}, Toshio Suzuki\thanks{Corresponding author. Partially supported by JSPS KAKENHI Grant Number 21K03340, 25K07105. E-mail {\tt toshio-suzuki@tmu.ac.jp}}
\\
{\small Department of Mathematical Sciences, Tokyo Metropolitan University}, \\ {\small 1-1 Minami-Osawa, Hachioji, Tokyo 1920397, Japan}
}

\maketitle

\begin{abstract}
Herein, we investigate the zero-error randomized complexity, which is the least cost against the worst input, of AND--OR tree computation by imposing various restrictions on the algorithm to find the Boolean value of the root of that tree and no restrictions on the tree shape. When a tree satisfies a certain condition regarding its symmetry, directional algorithms proposed by Saks and Wigderson (1986), special randomized algorithms, are known to achieve the randomized complexity. Furthermore, there is a known example of a tree that is so unbalanced that no directional algorithm achieves the randomized complexity (Vereshchagin 1998). In this study, we aim to identify where deviations arise between the general randomized Boolean decision tree and its special case, directional algorithms. We show that for any AND--OR tree, randomized depth-first algorithms, which form a broader class compared with directional algorithms, have the same equilibrium as that of the directional algorithms. Thus, we get the collapse result on equilibria inequalities that holds for an arbitrary AND--OR tree. This implies that there exists a case where even depth-first algorithms cannot be the fastest, leading to the separation result on equilibria inequality. Additionally, a new algorithm is introduced as a key concept for proof of the separation result. 

Keywords: 
AND--OR tree, depth-first algorithm, directional algorithm, randomized complexity. 

2000 MSC: 
Primary 68T20, 68Q17; Secondary 03D15, 91A60;

CCS Concepts: Theory of computation $\to$ Probabilistic computation; Oracles and decision trees; 
\end{abstract}

\tableofcontents

\section{Introduction}

\subsection{Background}  

AND--OR tree is a computation model of a \emph{read-once} Boolean function, that is, a function in which each of its variables appears exactly once in its formula. In other words, a tree $T$ is an AND-OR tree if each internal node is labeled with AND or OR, and each leaf is a different Boolean variable. We are interested in the \emph{decision tree complexity} (or \emph{deterministic complexity}) of a tree here, i.e., the computational cost is measured by the number of leaves the algorithm (i.e., the decision tree) has queried. $\alpha$ runs the algorithms computing the tree and $\omega$ runs the assignments to the tree, then the decision tree complexity is defined as follows:
\begin{equation}
\min_{\alpha} \max_{\omega} \mathrm{cost} (\alpha,\omega).
\end{equation}

We are interested in the case where we allow randomization on algorithms and/or assignments. The expected value of cost is then simply called cost. 
For a fixed tree $T$, we consider two types of equilibrium. These equilibria are good indicators of the complexity of the problem of finding the value of the root. 
One type of the equilibrium is the \emph{$($zero-error$)$ randomized complexity}, where $A$ runs randomized algorithms and $\omega$ runs (nonrandomized) assignments to the tree, and is defined as 
\begin{equation} \label{eq:defofr}
R(T) := \min_{A} \max_{\omega} \mathrm{cost} (A,\omega).
\end{equation}
The other equilibrium, which can be considered the dual concept of $R$, is the \emph{distributional complexity}, defined in \eqref{eq:deofop}. Here, $S$ runs randomized assignments and $\alpha$ runs  (nonrandomized) algorithms.
\begin{equation} \label{eq:deofop}
P(T) := \max_{S} \min_{\alpha} \mathrm{cost} (\alpha,S)
\end{equation} 

It is well known that the randomized complexity may be smaller than the deterministic complexity \cite{SW86}. In other words, randomization may improve efficiency. There is a considerable amount of prior research on separation between deterministic complexity and randomized complexity. In particular, an open problem in \cite{SW86} has been solved in mid 2010s, and remarkable developments continue to take place \cite{ABBLSS2017, MS2015, MRS2018}.

Note the differences in symbols in the literature. Our definition of $R$ is the same as \cite{SW86}. 
We later define $R_{0}$ and $R_{1}$ also in the style of \cite{SW86}; let us call this notation ``classical style''. On the other hand, in some recent literature \cite{ABBLSS2017, MS2015, MRS2018}, the symbols  $R$, $R_{0}$, and $R_{1}$ are used to mean something different from what they do in \cite{SW86}; let us call it ``modern style''. Table~\ref{table:symbol4r} provides summary of differences. In the row of ``zero error, the root is $i$'', 
the truth assignment is restricted to those with root value $i$, and the algorithm is zero-error regardless of the root value.  
This paper adopt classical style. 

\begin{table}[htp]
\begin{center}
\caption{Symbols for variants of randomized complexity \label{table:symbol4r}}
\begin{tabular}{|c|c|c|}
\hline
meaning & classical style (this paper) & modern style  
\\
\hline
zero error version & $R$ & $R_{0}$
\\
zeo error, the root is $0$ & $R_{0}$ & 
\\
zero error, the root is $1$ & $R_{1}$ & 
\\
bounded error version & & $R$
\\
one-sided error version & & $R_{1}$
\\
\hline
\end{tabular}
\end{center}
\end{table}%

It is reasonable to restrict ourselves to algorithms with alpha--beta pruning procedures: For each OR node $v$, we assume that once an algorithm finds a child node of $v$ has value 1 (true) then the algorithm does not search the rest of the descendants of $v$ further and finds that $v$ has value 1,  and we make similar assumptions for AND nodes. 
 
It is often the case that \emph{depth-first} algorithms are considered to examine essential instances without any unnecessary complication. A randomized algorithm of each type (general or depth-first) is defined as a probability distribution on each class of algorithms. Moreover, we consider a special type of randomized algorithms called \emph{directional algorithms} proposed by Saks and Wigderson \cite{SW86}. We refer to directional algorithms in this sense as RDA for technical reasons, which stands for randomized directional algorithm. 
The following is a somewhat lengthy definition, but it is important and should be stated in full here.

\begin{conv} \label{conv:shape}
Throughout this paper, no restrictions are placed on the shape of an AND--OR tree unless otherwise specified. 
An internal node may have arbitrary number ($\geq 2$) of child nodes. 
We also relax the constraints that the tree is alternating; to be more precise, each internal node of an AND--OR tree $T$ may be labeled with AND or OR regardless of its position, in other words, regardless of whether it is a root or not, and regardless of whether its parent node is labeled with AND or OR. 
\end{conv}

\begin{defi} \label{defi:symbols}
Suppose that $T$ is an AND-OR tree in the sense of Convention~\ref{conv:shape}.
\begin{enumerate}
\item \textbf{(Assignment)} The value of the root of $T$ is a Boolean function of the values of the leaves of $T$.
An \emph{assignment} to $T$ is a $\{0,1\}$ string of such input values to the leaves. 
More precisely, if $\ell_1,\ldots,\ell_n$ are the leaves of $T$, then an assignment is a mapping $\omega : \{\ell_1,\ldots,\ell_n\} \rightarrow \{0,1\}$. 

\item \textbf{(Assignment sets $\Omega (T)$ and $\Omega_{i} (T)$)} We denote the set of assignments to $T$ by $\Omega(T)$, and the set of assignments which gives the root value $i\in\{0,1\}$ by $\Omega_i(T)$.

\item \textbf{($\mathcal{A} (T)$, the deterministic algorithms)} 
A \emph{deterministic algorithm} $\alpha$ on $T$ is a Boolean decision tree finding the Boolean function described above. In particular, we only consider algorithms with \emph{alpha--beta pruning procedure}, which is described as follows. For any AND node [resp. OR node] $v$ and its children $v_1\ldots,v_n$, if it evaluates any $v_i$ with the value 0 [resp. 1], the algorithm immediately determines the value of $v$ as 0 [resp. 1] and omits the rest of the evaluation of $v$'s descendants. We let $\mathcal{A} (T)$ denote the set of all deterministic algorithms on $T$. 

\item \textbf{($\mathcal{A}_{\mathrm{DF}} (T)$, the depth-first algorithms)} A deterministic algorithm is said to be \emph{depth-first} if it satisfies the following. Whenever it starts to evaluate an internal node $v$, it does not query the leaves that are not descendants of $v$, until the value of $v$ is determined. We let $\mathcal{A}_{\mathrm{DF}} (T)$ denote the set of all deterministic depth-first algorithms on $T$. 

\item \textbf{($\mathcal{A}_{\mathrm{dir}} (T)$, the directional algorithms)} A deterministic depth-first algorithm is said to be \emph{directional} if the algorithm has a fixed order of leaf queries, and does not skip the queries unless in the case where it omits by alpha--beta pruning procedure explained above.
Note that in this paper, a directional algorithm is always depth-first. 
We let $\mathcal{A}_{\mathrm{dir}} (T)$ denote the set of all deterministic directional algorithms on $T$. 
Thus we have $\mathcal{A}_{\mathrm{dir}} (T) \subseteq \mathcal{A}_{\mathrm{DF}} (T) \subseteq \mathcal{A} (T)$. 
The two subset symbols are proper. The former will be seen later in Example~\ref{exam:df_not_dir}. 

\item \textbf{($\mathcal{D}(M)$, the distributions on a set $M$)} For a finite set $M$, let $\mathcal{D}(M)$ denote the set of probability distributions on $M$, that is,
\begin{equation}
\mathcal{D}(M) = \left\{ \delta : M \rightarrow [0,1] \ \left| \ \sum_{m\in M} \delta(m) = 1 \right. \right\}.
\end{equation}

\item \textbf{(Randomized algorithm)} A \emph{randomized algorithm} on $T$ is an element of $\mathcal{D}(\mathcal{A}(T))$, that is, a probability distribution on $\mathcal{A} (T)$.
For a randomized algorithm $A$ on $T$ and each $\alpha \in \mathcal{A} (T)$, $A(\alpha) $ denotes the probability of $\alpha$. 

\item \textbf{(Randomized depth-first algorithm,)} A \emph{randomized depth-first algorithms} is an element of $\mathcal{D}(\mathcal{A}_{\mathrm{DF}}(T))$, that is, a probability distribution on $\mathcal{A}_{\mathrm{DF}}(T)$. 

\item \textbf{($\mathfrak{S}_n$)} We let $\mathfrak{S}_n$ denote the $n$th symmetric group, that is, the set of all permutations on $\{ 1,\dots,n \}$. 

\item \textbf{(RDA)} A \emph{randomized directional algorithm in the sense of Saks and Wigderson}  (RDA, for short) is recursively defined. A randomized algorithm $A$ is an \emph{RDA}  if there exist a sequence of RDA $\langle A_{1}, \dots, A_{n} \rangle$ on the subtrees $T_{1}, \dots, T_{n}$ just under the root and a probability distribution $\pi$ on $n$th symmetric group $\mathfrak{S}_n$ with the following property. 
For each permutation $\sigma$ on $\{ 1,\dots,n \}$, with probability $\pi (\sigma)$, $A$ performs $A_{\sigma (1)}, \dots, A_{\sigma (n)} $ in this order.
\end{enumerate}
\end{defi}

\begin{rema} 
For the above-mentioned tree $T$, the set of all \emph{alpha--beta pruning algorithms} in the sense of Knuth and Moore (1975) \cite{KM75} equals $\mathcal{A}_{\mathrm{DF}} (T)$. 
\end{rema}

Inclusion relationships between  algorithm classes are as follows:

\begin{equation} \label{eq:inclusionrelation}
(\text{randomized algorithms}) \supseteq (\text{randomized depth-first algorithms}) 
\supseteq  (\text{RDA}) 
\end{equation}

For a fixed tree $T$, by restricting the type of randomized algorithms in \eqref{eq:defofr} to randomized depth-first algorithms (RDAs, repectively), we define $R_{\mathrm{DF}}(T)$ ($d(T)$, respectively). 
Precise definitions of $R_{\mathrm{DF}}(T)$ and $d(T)$ are given in Definition~\ref{defi:equi}.
We often consider dual complexity $P_{\mathrm{DF}} (T)$ by restricting the algorithm type in \eqref{eq:deofop} to depth-first ones. Note that the set of all randomized depth-first algorithms is convex in the following sense. For any randomized depth-first algorithms $A$ and $B$ and a nonnegative real number $p \leq 1$, the algorithm $p A + (1-p) B$ (the algorithm that works as $A$ with probability $p$ and as $B$ with probability $1-p$) is also a randomized depth-first algorithm. 
Using the properties of convex sets, we obtain an important property that the randomized complexity and its dual are equal. 

\smallskip

Yao's principle \cite{Ya77}: 
For any AND--OR tree $T$, we have $R(T)=P(T)$, and $R_{\mathrm{DF}}(T)=P_{\mathrm{DF}}(T)$  
(We will review the proof for this in Proposition~\ref{prop:vonneumanneq}). 

\smallskip

While RDA has the advantage of being easy to use via induction on tree height, it also has the drawback of not being convex. Thus, Yao's principle is not applicable for RDAs. 
We can define a minimum convex set of randomized algorithms containing all RDAs, which lies between the randomized depth-first algorithms and RDAs in the inclusion relationships \eqref{eq:inclusionrelation}. 
We let ``a randomized directional algorithm in the broad sense'' to denote an element of this convex set. 

\begin{defi} \label{defi:symbols2}
\textbf{(Randomized directional algorithm in the broad sense)} 
A \emph{randomized directional algorithm in the broad sense} is an element of $\mathcal{D}(\mathcal{A}_{\mathrm{dir}}(T))$, that is, a probability distribution on $\mathcal{A}_{\mathrm{dir}} (T)$. 
\end{defi}

$R_{\mathrm{dir}} (T)$ and $P_{\mathrm{dir}} (T)$ are defined in the same way as above for randomized directional algorithms in the broad sense, and the Yao's principle holds in this case: $R_{\mathrm{dir}} (T) = P_{\mathrm{dir}} (T)$. 

Previous studies have reported various results regarding algorithms of the above-mentioned types. 
A \emph{balanced tree} is a tree such that the internal nodes of the same depth (distance from the root) have the same number of child nodes, and all leaves have the same depth.  
In \cite{LT07,SN12,P16}, it was shown that for every complete binary tree (or more generally, a balanced tree), there exists a unique randomized assignment that attains $P(T)$ (the same applies to $P_{\mathrm{DF}}(T)$), but infinitely many assignments attains $P_{\mathrm{dir}}(T)$. \cite{SN15}, \cite{P17}, and \cite{P22} investigated the case in which randomized assignments are restricted to ones with independent probabilities for each leaves, and showed that directional algorithms are sufficient to achieve the minimum cost for such assignments. 

\cite{SW86} and \cite{KSS22} revealed some important facts related to our result; when we consider a tree with certain mild condition applied to its symmetry (which we call the \emph{weak-balance condition}), $R=d$ holds.

\begin{defi} \cite{KSS22} (implicitly in \cite{SW86}) \label{defi:weakly-balanced}
Let $T$ be an AND--OR tree. For each node $v$ of $T$, let $T_v$ be the subtree of $T$ where the root node is $v$, and let $R_i(T_v)$ be the randomized complexity of $T_v$ with respect to the assignments setting the value of $v$ to $i$. $T$ is said to be \emph{weakly balanced} if the following holds: For each internal node $v$ of $T$ and its children $v_1,\ldots,v_n$, 
if $v$ is an AND (OR, respectively) node, then $R_0(T_{v_j}) \leq R_0(T_{v_k}) + R_1(T_{v_j})$ ($R_1(T_{v_j}) \leq R_1(T_{v_k}) + R_0(T_{v_j})$, respectively) for each $j,k\in\{1,\ldots,n\}$.
\end{defi}

It is known that every complete binary tree and balanced tree is weakly balanced. 
The following fact is significant. A more concise proof is provided in \cite{It24}. 

\begin{theo} \cite[Theorem 2.1]{KSS22} (see also \cite[Lemma 5.1]{SW86}) \label{theo:weaklybalanced}
Suppose an AND--OR tree $T$ is weakly balanced. Then $R(T)=d(T)$.
\end{theo}

On the other hand, there exists an example in which the tree is utterly unbalanced, and thus, this equality does not hold.

\begin{exam} \cite{Ve98} (see also \cite{Am11}) \label{exam:vereshchagin}
There exists an AND--OR tree $V$ where
\[
R(V) \leq 51, \ d(V) = \frac{33525}{640} = 52.38\ldots
\]
\end{exam}

\subsection{Our goal: Collapse and separation} 

Our main motivations for focusing on the directional algorithms is precisely because Example~\ref{exam:vereshchagin} holds. 
Considering Example~\ref{exam:vereshchagin}, it is natural to ask where does the gap lies between $R$ and $d$. 

Incidentally, in Saks-Wigderson \cite{SW86} and in Vereshchagin \cite{Ve98}, 
there appears to be an implicit consensus that instead of considering $R_\mathrm{DF}$, it is sufficient to consider $d$. This would be based on an intuition that it is not surprising for $R_\mathrm{DF}$ and $d$ to be equal. 

\begin{obse} \label{obse:intuitive}
Suppose that a randomized depth-first algorithm probes subtrees $T_{1}$ and $T_{2}$ under a node $v$ of a  tree. The intuitive observation would be that, since probing $T_{1}$ and probing $T_{2}$ are ``independent'', when performing depth-first-probing $T_{1}$, there is no need to consider $T_{2}$ in choosing an ``optimal'' algorithm to probe $T_{1}$, nor is it necessary to consider the history prior to starting the probe of $T_{1}$. 
\end{obse}

It must be noted that Observation~\ref{obse:intuitive} is not true in the case where we consider arbitrary depth-first algorithms. The intuition is false because the ``independence'' fails in the following two ways. 

Let $A$ be a randomized depth-first algorithm. Firstly, when we consider arbitrary randomized depth-first algorithms, probabilistic independence between the sub-algorithms of $A$ does not hold by its definition. 

Secondly, the sub-algorithms may not be independent in some cases, even if $A$ is deterministic. Suppose that $A$ is deterministic, and it evaluates $T_1$ first and then $T_2$. As we will see in Example~\ref{exam:df_not_dir}, It is possible that the query order in $T_2$ changes depending on the query result of $T_1$. Thus, the movement of a depth-first algorithm in subtrees is not always independent in the sense of query order, either. 

Furthermore, there is a more important point that must be noted: The hierarchy of equilibrium values collapses when we make mild assumptions about the symmetry of the tree. 

It is immediate that $R(T) \leq R_{\mathrm{DF}}(T) \leq R_{\mathrm{dir}}(T)\leq d(T)$. Under a certain hypothesis on $T$ (weakly balanced in the sense of Definition~\ref{defi:weakly-balanced}), it is shown that $R(T) = d(T)$ (Theorem~\ref{theo:weaklybalanced}). 

Various types of trees, including binary trees, are weakly-balanced. So we must ask: In Observation~\ref{obse:intuitive}, have we not unconsciously made assumptions about the shape of the tree?

So we paraphrase the previous question, where does the gap lies between $R$ and $d$, in a more particular form. What assumptions about the symmetry of the tree are sufficient conditions for deriving $R_\mathrm{DF}=d$?
Our conclusion is that no assumptions need be made. 

We show a collapse on equilibria 
\begin{equation} \label{eq:rdftisdt}
R_{\mathrm{DF}}(T)=d(T)
\end{equation}
\noindent
for all AND--OR tree $T$, regardless of the assumption of weakly balanced. 
The randomized complexity with respect to the depth-first algorithms equals that with respect to the directional algorithms (RDAs). Here, no shape constraints are imposed on the tree $T$. Any node of $T$ can have any number (${\geq}2$) of child nodes, and there can be a mixture of AND and OR nodes among the child nodes. \eqref{eq:rdftisdt} is a strict equality without Landau symbol $O$.  
Thus, we obtain a separation result on the equilibria 
\begin{equation}
R(T) < R_{\mathrm{DF}}(T)
\end{equation}
\noindent
for a certain $T$: The gap lies between the general algorithms and the depth-first algorithms. In other words, our result implies that there exists a case where a depth-first algorithm cannot attain the randomized complexity $R$.
This study thus provides a deeper understanding of the depth-first algorithm for AND--OR trees.

\subsection{Technical obstacles}

The ideal scenario for proof is as follows.

\begin{enumerate}
\item[(i)] Consider an algorithm $A$ that achieves $R_\mathrm{DF}(T)$. 
\item[(ii)] By adopting a certain induction hypothesis, it is shown that there exists a good algorithm $B_{i}$ for each subtree $T_{i}$ just under the root of $T$. 
\item[(iii)] Replace the behavior of $A$ in subtree $T_{i}$ with algorithm $B_{i}$ for each $i$. Then, let $A^{\prime}$ be the resulting algorithm. 
\item[(iv)] We want to show that $A^{\prime}$ is an RDA, and that the cost of $A^{\prime}$ is lower than or equal to the cost of $A$. 
\item[(v)] We prove that $R_\mathrm{DF}(T) = d(T)$ using Yao's principle if necessary. 
\end{enumerate}

However, when it is time to implement this scenario in a proof, the following difficulties arise. 

\begin{description}  
\item[(Obstacle 1)]
As mentioned earlier, we cannot apply the Yao's principle to RDAs. 
\item[(Obstacle 2)]
When $A$ in (i) is not an RDA, the meaning of the operations described in (iii) is unclear. 
\end{description}

\subsection{Ways to overcome difficulties}

\textbf{Big-picture strategy:} ~ 
To overcome Obstacle 1, we divide the problem of showing $R_{\mathrm{DF}}(T)=d(T)$ into two problems.
\begin{itemize}
\item Problem 1: $R_{\mathrm{dir}}(T) =d(T)$,
\item Problem 2: $R_{\mathrm{DF}}(T)=R_{\mathrm{dir}}(T)$. 
\end{itemize}
The set of all randomized directional algorithms in the broad sense equals the convex hull of the set of all RDAs. 
We can apply Yao's principle to randomized directional algorithms in the broad sense. 

\textbf{Problem 1:} ~  
In Section~\ref{sec:problem1}, we solve Problem 1. To overcome Obstacle 2 in Problem 1, we provide precise definition on replacement of sub-algorithm. 
Suppose that $T$ is an AND-OR tree and that $T_{1}, \dots, T_{n}$ are subtrees just under the root. 
Suppose that $\alpha$ is a directional algorithm of $T$, $k \in \{1,\ldots,n\}$, and $Y_{(k)}$ is an RDA of $T_k$.  
In Definition~\ref{defi:rep_subalgo01}, we give a precise definition of  a new algorithm $\alpha [ Y_{(k)} / T_k ]$ by replacing the behavior on $T_{k}$ of $\alpha$ with $Y_{(k)}$. 
In Theorem~\ref{theo:rdiri_is_di} (2), we show the following.  

Let $X$ be a randomized directional algorithm on $T$ in the broad sense. Then there exists an RDA $Z$ on $T$ with the following properties.
\begin{align*}
\max_{\xi \in \Omega(T)} \mathrm{cost}(Z, \xi) \leq \max_{\xi \in \Omega(T)} \mathrm{cost}(X, \xi) 
\\
\max_{\xi \in \Omega_0(T)} \mathrm{cost}(Z, \xi) \leq \max_{\xi \in \Omega_0(T)} \mathrm{cost}(X, \xi) 
\\
\max_{\xi \in \Omega_1(T)} \mathrm{cost}(Z, \xi) \leq \max_{\xi \in \Omega_1(T)} \mathrm{cost}(X, \xi) 
\end{align*}

Theorem~\ref{theo:rdiri_is_di} (2) is shown by means of the following three. 
\begin{itemize}
\item \cite{SW86} A known result that there exists an RDA that attains both $d_{0}$ and $d_{1}$. 
\item A precise evaluation of cost of $\alpha [ Y_{(k)} / T_k ]$. 
\item Induction on the height of the tree.
\end{itemize}

The item 1 is stated in \cite[Lemma 3.1]{SW86} without proof. To keep the paper self-contained, we give a proof the item 1.  
By Theorem~\ref{theo:rdiri_is_di} (2), we show $R_{\mathrm{dir},i}(T) = d_{i}(T)$. In this way, Problem 1 will be  solved in Section~\ref{sec:problem1}.


\textbf{Problem 2:} ~  
Problem 2 is more difficult to solve. In Section~\ref{sec:problem2}, to solve Problem 2, we investigate the case when $A$ in (i) is a deterministic depth-first algorithm $\alpha$. Instead of $A^{\prime}$ in (iii), we elaborately define a randomized directional algorithm in the broad sense $B_{\alpha,\delta}$, which is a key concept in this study. Figure~\ref{fig:dependency00} illustrates interdependence among concepts necessary to define $B_{\alpha,\delta}$. 

\begin{figure}[h] 
  \centering
  \includegraphics[width=110mm,bb=0 0 1422 663]{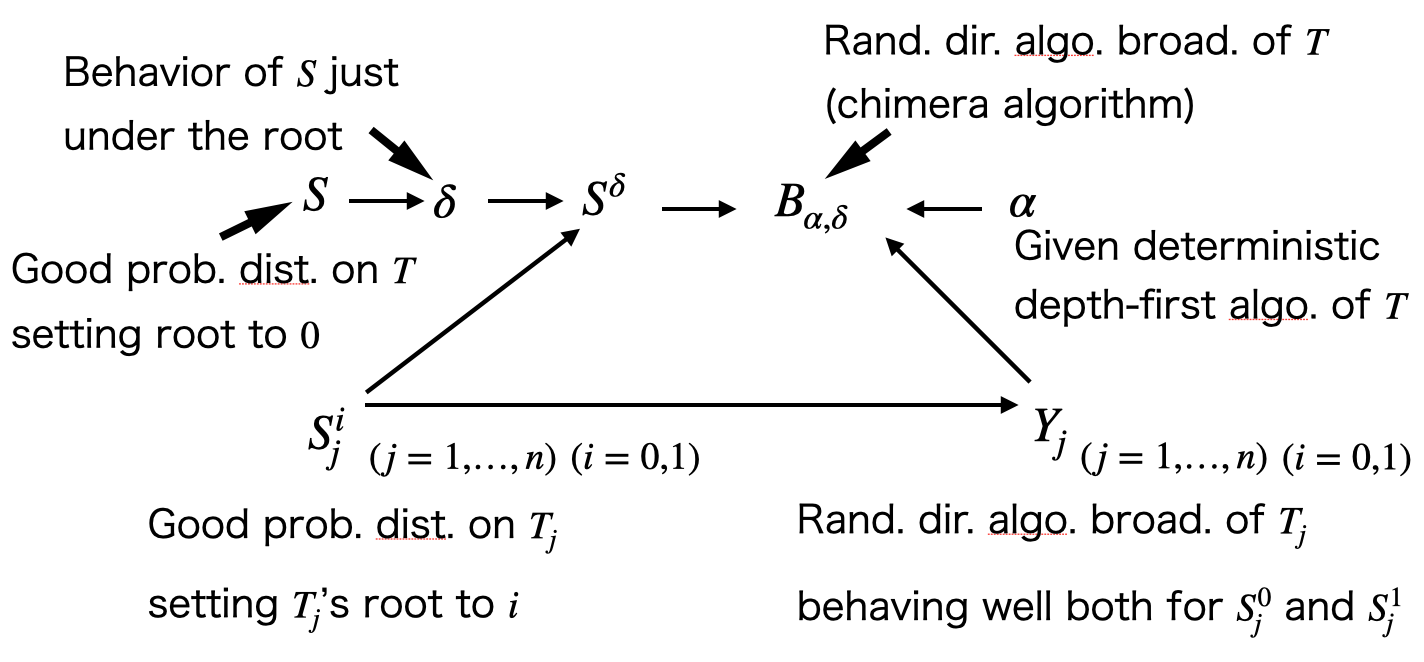}
  \caption{Interdependence of definitions, simplified version}
  \label{fig:dependency00}
\end{figure}

\begin{itemize}
\item $\alpha$ is a given deterministic depth-first algorithm of $T$.
\item $S$ is a good probability distribution on $T$ setting value of root $0$. 
\item $\delta$ is the behavior of $S$ just under the root. Thus, $\delta$ is a probability distribution on the truth assignments to nodes just under the root. 
\item For each $j = 1,\dots,n$ and $i=0,1$, $S^{i}_{j}$ is a good probability distribution that sets value of the root of $T_{j}$ to $i$. 
\item $S^{\delta}$ is a probability distribution on $T$ made by combining $\delta$ and $S^{i}_{j}$. 
At the depth-1 level of $T$, $S^{\delta}$ behaves $S$-like. 
On each $T_{j}$, $S^{\delta}$ behaves $S^{i}_{j}$-like. 
\item $Y_{j}$ is a randomized directional algorithm in the broad sense of $T_{j}$ such that $Y_{j}$ behaves well both for $S^{0}_{j}$ and $S^{1}_{j}$. 
\item $B_{\alpha,\delta}$ is a  randomized directional algorithm in the broad sense of $T$. 
\begin{itemize}
\item At the depth-1 level of $T$, $B_{\alpha,\delta}$ behaves similarly to the movement of $\alpha$ with respect to $S^{\delta}$.
\item On each $T_{j}$, $B_{\alpha,\delta}$ behaves $Y_{j}$-like. 
\end{itemize}
\end{itemize}

The nickname ``chimera'' for algorithm $B_{\alpha,\delta}$ arises from the fact that $B_{\alpha,\delta}$ is composed of parts from several other algorithms and probability distributions as well as $\alpha$. 
We show Inequality \eqref{eq:lboundalphasdelta2}:
\[
\mathrm{cost}(B_{\alpha,\delta},S^{\delta}) \leq \mathrm{cost}(\alpha, S^{\delta}) 
\]
This inequality is a key to prove $R_{\mathrm{DF},i}(T)=R_{\mathrm{dir},i}(T)$. 
In this way, Problem 2 will be  solved in Section~\ref{sec:problem2}.
Figure~\ref{fig:dependency00} is a simplified version of Figure~\ref{fig:dependency01} in subsection~\ref{subsec:4outline}. 

In Section~\ref{sec:def}, we give the basic definitions and list some fundamental facts regarding equilibrium values. 
In Section~\ref{sec:remark}, we summarize our results and describe some immediate consequences along with those reported in previous studies. 
In Section~\ref{sec:appendix}, we provide rigorous proofs of some facts used in other sections.
Looking back at our proof, we can say that Observation 1.8 oversimplifies the facts. 
It would be an interesting direction of future research to extend our proof to computational models other than AND-OR trees, in particular trees composed of threshold functions or majority functions \cite{CP22}.

\section{Definitions and Preliminaries}
\label{sec:def}  

In this section, we will define some basic concepts and see some fundamental facts on AND--OR trees. See also \cite[Chapter 12]{AB09} and \cite{Su17a} for helpful examples.

Saks and Wigderson \cite{SW86} defined RDA in order to get a recurrence formula of randomized complexity. 
In general, a probability distribution on the deterministic directional algorithms is not necessarily an RDA .

\begin{exam} \label{exam:df_not_dir}
 \textbf{(A depth-first algorithm that is not directional)}
We observe a depth-first algorithm that is not directional. 
Let $T$ be a complete binary tree of height 2 such that the root is labeled with AND, and each of child nodes $v_{1},v_{2}$ of the root is labeled with OR. Let $x_{1}, x_{2}$ be the leaves that are child nodes of $v_{1}$, 
and $x_{3}, x_{4}$ be the leaves that are child nodes of $v_{2}$. 
Let $\alpha \in \mathcal{A}_{\mathrm{DF}} (T)$ be the algorithm whose priority of probing leaves is as follows. 
\begin{itemize}
\item $\alpha$ begins with probing $x_{1}$. 
\item If $x_{1}$ has value $1$ then $\alpha$ probes $x_{3}$. If $x_{3}$ has value $1$ then $\alpha$ returns $1$ and halts. Otherwise, $\alpha$ returns value of $x_{4}$ and halts. 
\item If $x_{1}$ has value $0$ then $\alpha$ probes $x_{2}$. If $x_{2}$ has value $0$ then $\alpha$ returns $0$ and halts. Otherwise, $\alpha$ probes $x_{4}$. If $x_{4}$ has value $1$ then $\alpha$ returns $1$ and halts. Otherwise, $\alpha$ returns value of $x_{3}$ and halts.  
\end{itemize}

For truth assignment $\omega_{1}=(1,1,0,0)$ (that is, $\omega_{1}$ assigns $(1,1,0,0)$ to $(x_{1},x_{2},x_{3},x_{4})$), $\alpha$ probes $T$ in the order of $x_{1},x_{3},$ and $x_{4}$. 
For truth assignment $\omega_{2}=(0,1,0,0)$, $\alpha$ probes $T$ in the order of $x_{1},x_{2},x_{4}$ and $x_{3}$. 
Since priority of $x_{3}$ and $x_{4}$ is not fixed, $\alpha$ is not directional.  

It holds that $\alpha \in \mathcal{A}_\mathrm{DF} (T) \setminus \mathcal{A}_\mathrm{dir} (T)$. 
At the same time, we have 
$\alpha \in \mathcal{D}( \mathcal{A}_\mathrm{DF} (T) ) \setminus \mathcal{D}( \mathcal{A}_\mathrm{dir} (T)$ ). 
\end{exam}

\begin{exam} \label{exam:dir_not_rda} 
\textbf{(A randomized directional algorithm in the broad sense that is not an RDA)} 
We observe a randomized directional algorithm in the broad sense that is not an RDA. 
Let $T$ be as in Example~\ref{exam:df_not_dir}. 
Let $\alpha \in \mathcal{A}_{\mathrm{dir}} (T)$ be the algorithm whose priority of probing leaves is in the order of $x_{1},x_{2},x_{3},$ and $x_{4}$. 
Let $\beta \in \mathcal{A}_{\mathrm{dir}} (T)$ be the algorithm whose priority of probing leaves is in the order of $x_{2},x_{1},x_{4},$ and $x_{3}$. 
Let $A \in \mathcal{D}( \mathcal{A}_{\mathrm{dir}} (T) )$ be the randomized algorithm such that $A$ works as $\alpha$ with probability $1/2$ and as $\beta$ with probability $1/2$. 
Let ``$x_{i} \to x_{j}$'' denote the event that $x_{i}$ is probed before $x_{j}$.
Then we have the following, where $\mathrm{Pr}_{A}$ denotes probability with respect to $A$. 

\begin{align*}
& \mathrm{Pr}_{A} [ x_{3} \to x_{4} | x_{1} \to x_{2} \land v_{1} \text{ has value } 1 ] (= 1)
\\
\ne & \mathrm{Pr}_{A} [ x_{3} \to x_{4} | x_{2} \to x_{1} \land v_{1} \text{ has value } 1 ] (= 0)
\end{align*}

Therefore $A$ is not an RDA. 
\end{exam}

\begin{defi} \label{defi:equi}
Suppose that $T$ is an AND-OR tree in the sense of Convention~\ref{conv:shape}. 
\begin{enumerate}
\item Saks-Wigderson's $d$ is defined as follows \cite{SW86}.
\begin{equation}
d(T) := \min_{A \text{:RDA}} ~ \max_{\omega\in\Omega(T)} \mathrm{cost} (A, \omega)
\end{equation}
Here $A$ runs over the RDAs on $T$. $\mathrm{cost} (A, \omega)$ denotes expected value for a particular  $A$, in other words, 
 $\sum_{\alpha} A(\alpha) \mathrm{cost} (\alpha, \omega)$, where $\alpha$ runs over $\mathcal{A} (T)$ and $A(\alpha)$ is the probability of $A$ being $\alpha$. 
\item 
\begin{equation}
R_{\mathrm{dir}}(T) := \min_{A\in\mathcal{D}(\mathcal{A}_{\mathrm{dir}}(T))} ~ \max_{\omega\in\Omega(T)} \mathrm{cost} (A, \omega) 
\end{equation}
\item 
\begin{equation}
R_{\mathrm{DF}}(T) := \min_{A\in\mathcal{D}(\mathcal{A}_{\mathrm{DF}}(T))} ~ \max_{\omega\in\Omega(T)} \mathrm{cost} (A, \omega)
\end{equation}
\item For each $i \in \{ 0,1 \}$, we define $d_{i}$, $R_{\mathrm{dir},i}$ and $R_{\mathrm{DF},i}$ as follows. Note that the $\omega$ in each definition runs over the assignments which gives the root value $i$.
\begin{align}
d_{i}(T) := & \min_{A \text{:RDA}} ~ \max_{\omega\in\Omega_i(T)} \mathrm{cost} (A, \omega)
\\
R_{\mathrm{dir},i}(T) := & \min_{A\in\mathcal{D}(\mathcal{A}_{\mathrm{dir}}(T))} ~ \max_{\omega\in\Omega_i(T)} \mathrm{cost} (A, \omega)
\\
R_{\mathrm{DF},i}(T) := & \min_{A\in\mathcal{D}(\mathcal{A}_{\mathrm{DF}}(T))} ~ \max_{\omega\in\Omega_i(T)} \mathrm{cost} (A, \omega)
\end{align}
\item
\begin{equation}
P_{\mathrm{dir}}(T) := \max_{S \in \mathcal{D}(\Omega(T))} \min_{\alpha \in \mathcal{A}_{\mathrm{dir}}(T)} \mathrm{cost} (\alpha,S)
\end{equation}
\item
\begin{equation}
P_{\mathrm{DF}}(T) := \max_{S \in \mathcal{D}(\Omega(T))} \min_{\alpha \in \mathcal{A}_{\mathrm{DF}}(T)} \mathrm{cost} (\alpha,S)
\end{equation}
\item
For each $i \in \{ 0,1 \}$, we define $P_{\mathrm{dir},i}$ and $P_{\mathrm{DF},i}$ as follows.
\begin{align}
P_{\mathrm{dir},i}(T) := \max_{S \in \mathcal{D}(\Omega_i(T))} \min_{\alpha \in \mathcal{A}_{\mathrm{dir}}(T)} \mathrm{cost} (\alpha,S)
\\
P_{\mathrm{DF},i}(T) := \max_{S \in \mathcal{D}(\Omega_i(T))} \min_{\alpha \in \mathcal{A}_{\mathrm{DF}}(T)} \mathrm{cost} (\alpha,S)
\end{align}
\item
We say $S \in \mathcal{D}(\Omega(T))$ {\it achieves} $P(T)$ if we have
\begin{equation}
\min_{ \alpha \in \mathcal{A}(T) } \mathrm{cost} (\alpha, S) = P(T).
\end{equation}
We say $A\in\mathcal{D}(\mathcal{A}(T))$ achieves $R(T)$ if we have
\begin{equation}
\max_{\omega\in\Omega(T)} \mathrm{cost}(A,\omega) = R(T).
\end{equation}
Similarly, we say $S$ achieves $P_i(T), P_{\mathrm{dir}}(T)$ or $A$ achieves $R_i(T),R_{\mathrm{dir}}(T)$ and so on, when they satisfy the corresponding conditions. Whenever we say $S$ achieves $P_i(T)$ ($A$ achieves $R_{\mathrm{dir}}(T)$, respectively), we assume that $S \in \mathcal{D}(\Omega_{i} (T))$ ($A \in \mathcal{D} (\mathcal{A}_{\mathrm{dir}}(T)$), respectively).
\end{enumerate} 
\end{defi}

\begin{defi} \label{defi:vonneumanneq}
We define von Neumann-type equilibria as follows.
\begin{enumerate}
\item
\begin{align*} \label{eq:vonneumanneq01}
\overline{R}_{\mathrm{dir}} (T):
= 
&\min_{A \in \mathcal{D}(\mathcal{A}_\mathrm{dir}(T))} \max_{S \in \mathcal{D}(\Omega(T))} \mathrm{cost} (A,S)
\\
\overline{P}_{\mathrm{dir}} (T):
= &\max_{S \in \mathcal{D}(\Omega(T))} \min_{A \in \mathcal{D}(\mathcal{A}_\mathrm{dir}(T))} \mathrm{cost} (A,S)
\end{align*}
\item We define $\overline{R}_{\mathrm{dir},i} (T)$ and $\overline{P}_{\mathrm{dir},i} (T)$ in the same way as in Definition \ref{defi:equi} by putting a constraint on $S$ that the root of $T$ has value $i$. 
\item For depth-first algorithms, we define
$\overline{R}_{\mathrm{DF}} (T)$, $\overline{P}_{\mathrm{DF}} (T)$, 
$\overline{R}_{\mathrm{DF},i} (T)$, and $\overline{P}_{\mathrm{DF},i} (T)$
in the same way. 
\end{enumerate}
\end{defi}

\begin{prop} \label{prop:vonneumanneq}
\hspace{2em} 
\begin{enumerate}
\item $($von Neumann$)$

It holds that 
$\overline{R}_{\mathrm{dir}} (T) = \overline{P}_{\mathrm{dir}} (T)$, 
$\overline{R}_{\mathrm{dir},i} (T) = \overline{P}_{\mathrm{dir}.i} (T)$, 
$\overline{R}_{\mathrm{DF}} (T) = \overline{P}_{\mathrm{DF}} (T)$, 
and 
$\overline{R}_{\mathrm{DF},i} (T) = \overline{P}_{\mathrm{DF}.i} (T)$.  
\item
Suppose that $A \in \mathcal{D}(\mathcal{A} (T))$, and that $\Omega^{\prime}$ is either 
$\Omega(T) $ or $\Omega_{i}(T) $.
Then we have the following.
\begin{equation}
\max_{S \in \mathcal{D} (\Omega^{\prime}) } \mathrm{cost} (A,S)
=
\max_{\omega \in \Omega^{\prime} } \mathrm{cost} (A,\omega)
\end{equation}
\item
Suppose that $S \in \mathcal{D}(\Omega(T))$, and that $\mathcal{A}^{\prime}$ is one of 
$ \mathcal{A}(T) $, 
$ \mathcal{A}_\mathrm{DF} (T) $, or 
$\mathcal{A}_\mathrm{dir} (T) $. 
Then we have the following.
 
\begin{equation}
\min_{A \in \mathcal{D}(\mathcal{A}^{\prime})} \mathrm{cost} (A,S)
=
\min_{\alpha \in \mathcal{A}^{\prime}} \mathrm{cost} (\alpha,S)
\end{equation}
\item
It holds that $\overline{R}_{\mathrm{dir}} (T) = R_{\mathrm{dir}} (T)$. 
The same holds for $\overline{R}_{\mathrm{dir},i} (T)$, 
$\overline{R}_{\mathrm{DF}} (T)$, and 
$\overline{R}_{\mathrm{DF},i} (T)$.  
The same holds for $P$ in place of $R$. 
\item $($Yao's principle$)$ 
It holds that 
$R_{\mathrm{dir}} (T) = P_{\mathrm{dir}} (T)$, 
$R_{\mathrm{dir},i} (T) = P_{\mathrm{dir}.i} (T)$, 
$R_{\mathrm{DF}} (T) = P_{\mathrm{DF}} (T)$, 
and 
$R_{\mathrm{DF},i} (T) = P_{\mathrm{DF}.i} (T)$.  
\end{enumerate}
\end{prop}

\begin{proof} 
(1) directly follows from von Neumann's minimax theorem. 

(2). The inequality $\geq$ is obvious. Let $S_{*}$ be an element of $\mathcal{D} (\Omega^{\prime}) $ such that 
\[
\mathrm{cost} (A,S_{*})
= \max_{S \in \mathcal{D} (\Omega^{\prime}) } \mathrm{cost} (A,S).
\]
Let $\omega_{*}$ be an element of $\Omega^{\prime} $ such that it maximizes 
$\mathrm{cost} (A,\omega)$ among all $\omega \in \Omega^{\prime} $ such that 
$S_{*}(\omega)>0$. Then $\mathrm{cost} (A,S_{*}) = \sum_{\omega \in  \Omega^{\prime}} S_{*}(\omega) \mathrm{cost} (A,\omega) \leq \mathrm{cost} (A,\omega_{*}) \leq \max_{\omega \in \Omega^{\prime} } \mathrm{cost} (A,\omega)$. Thus the inequality $\leq$ holds. 

(3) is shown in the same way as (2). 

(4) is shown by assertions (2) and (3).

(5) is Yao's principle. It is shown by assertions (1) and (4).
\end{proof}

\section{Randomized directional algorithms} \label{sec:problem1}

In this section, we are going to prove that $R_{\mathrm{dir},i}(T) = d_{i}(T)$. 
Similarly as in Section 2, throughout this section, let $T$ be an AND--OR tree whose root has $n$ child nodes $v_{1}, \dots, v_{n}$. Let $T_{1}, \dots, T_{n}$ be the subtrees whose roots are $v_{1}, \dots, v_{n}$ respectively. 
For each $\omega\in\Omega(T)$, $\alpha \in \mathcal{A}_{\mathrm{dir}}(T)$, and an RDA $X$, symbols $\omega_k$, $\alpha_{k}$, and $X_{k}$ denote the $T_{k}$-parts of $\omega$, $\alpha$, and $X$, respectively.

In the case when $X$ is an RDA, the meaning of operation to replace $X_{k}$ by an RDA $Y_{(k)}$ on $T_{k}$ is clear. Indeed, Lemma~\ref{lemm:sw86lem3_1} will be proved in this way. We want to investigate a similar operation in the case when $X$ is a randomized directional algorithm in the broad sense. In this case, we must be more careful than the case when $X$ is an RDA, because the $T_{k}$-part of $X$, given probabilistically, depends on the other parts of $X$. 
In subsection~\ref{subsection:rda1}, we investigate the operation of replacing a subalgorithm of a directional algorithm in the broad sense, and show equations and inequalities on computational cost. 
By means of these equations and inequalities, we show $R_{\mathrm{dir},i}(T) = d_{i}(T)$ in subsection~\ref{subsection:rdavsrdir}. 

\subsection{Replacement of sub-algorithms} \label{subsection:rda1}

The following result is written in \cite{SW86} without a proof, and we include it here. 

\begin{lemm} \emph{\cite[Lemma 3.1]{SW86}} \label{lemm:sw86lem3_1} 
There exists an RDA on $T$ that attains both $d_{0}(T)$ and $d_{1}(T)$.
\end{lemm}

\begin{proof} (sketch)
We show it by induction on the height of a tree. The height 0 case is trivial. 
In the induction step, we may assume the root is labeled with AND without loss of generality. 
For each subtree $T_{j}$, by means of induction hypothesis, we chose an RDA which achieves both $d_{0} (T_{j})$ and $d_{1} (T_{j})$. 
We chose a distribution on the permutations (of child nodes of the root) so that the RDA made of this distribution and the RDAs chosen above achieves $d_{0} (T)$. 
It is not difficult to verify that this RDA also achieves $d_{1}(T)$. 
For a more detailed proof, see Appendix. 
\end{proof}

Let $\alpha$ be a directional algorithm of $T$, $k \in \{1,\ldots,n\}$, and $Y_{(k)}$ an RDA of $T_k$. Intuitively, we want to define a new algorithm $\alpha [ Y_{(k)} / T_k ]$ by replacing the behavior on $T_{k}$ of $\alpha$ with $Y_{(k)}$. The following gives the precise definition. 

\begin{defi} \label{defi:rep_subalgo01} 
Suppose that $X$ belongs to $\mathcal{D}(\mathcal{A}_\mathrm{dir} (T))$, $k \in \{ 1,\dots,n \}$, and that $Y_{(k)}$ is an RDA on $T_{k}$. 
\begin{enumerate}
\item For a directional algorithm $\alpha \in \mathcal{A}_\mathrm{dir} (T)$ and a permutation $\sigma\in\mathfrak{S}_n$, we say $\alpha$ is \emph{order} $\sigma$ if the order of evaluation of $v_1,\ldots,v_n$ by $\alpha$ follows $\sigma$. We let $\sigma_{\alpha} \in \mathfrak{S}_{n}$ denote the permutation such that $\alpha$ is order $\sigma_{\alpha}$. 

\item We define a binary relation $\alpha \equiv \beta ~ ({\hat{k}}) $ for $\alpha, \beta \in \mathcal{A}_\mathrm{dir} (T)$ as follows. 
\begin{align} 
&\alpha \equiv \beta ~ ({\hat{k}}) 
\notag
\\
\iff &\text{$\sigma_{\alpha} = \sigma_{\beta}$, and $\alpha_{j} = \beta_{j}$ for each $j\in \{ 1,\dots, n\} \setminus \{ k \}$.} \label{eq:equivhat01}
\end{align}
Note that for a fixed $\alpha$, whether $\beta$ satisfies the property \eqref{eq:equivhat01} does not depend on $\beta_{k}$. 
\item
For $\alpha \in \mathcal{A}_\mathrm{dir} (T)$, we let symbol $\mathrm{Pr}_{X}[X \equiv \alpha ~ (\hat{k})] $ denote the probability of randomly chosen $\beta \in \mathcal{A}_\mathrm{dir} (T)$ satisfies $\alpha \equiv \beta ~ (\hat{k})$, 
where probability is measured by $X$ (we let $\mathrm{Pr}_{X}^\mathrm{dir}$ denote this probability measure on $\mathcal{A}_\mathrm{dir} (T)$, and omit the superscript $\mathrm{dir}$), that is,
\begin{equation}
\mathrm{Pr}_{X}[X \equiv \alpha ~ (\hat{k})] 
:= \sum_{\substack{\beta \in \mathcal{A}_\mathrm{dir} (T) \\ \alpha \equiv \beta ~ (\hat{k})}} X(\beta).
\end{equation}
\item 
Suppose $\alpha \in \mathcal{A}_\mathrm{dir} (T)$. 
We define $\alpha [Y_{(k)}/T_{k}] \in \mathcal{D}(\mathcal{A}_\mathrm{dir} (T))$ by defining its probability for each $\gamma \in \mathcal{A}_\mathrm{dir}(T)$ as follows. 
\begin{equation} \label{eq:alphayktk01}
\alpha [Y_{(k)}/T_{k}] (\gamma) = 
\begin{cases}
Y_{(k)} (\gamma_{k}) \quad & \text{if $\gamma \equiv \alpha ~ ({\hat{k}}) $}
\\
0 & \text{otherwise}
\end{cases}
\end{equation}
We can see that $\alpha [Y_{(k)}/T_{k}]$ is well-defined as a probabilistic distribution as follows. 
\begin{align}
\sum_{\gamma \in \mathcal{A}_\mathrm{dir} (T) } \alpha [Y_{(k)}/T_{k}] (\gamma)
=&
\sum_{\substack{\gamma \in \mathcal{A}_\mathrm{dir} (T) \\ \gamma \equiv \alpha (\hat{k})}} \alpha [Y_{(k)}/T_{k}] (\gamma)
\notag
\\
=&
\sum_{\gamma_{(k)} \in \mathcal{A}_\mathrm{dir} (T_{k}) } Y_{(k)} (\gamma_{(k)})
\notag
\\
=&1
\end{align}
\end{enumerate}
\end{defi}

\begin{defi} \label{defi:rep_subalgo02} 
Suppose $X \in \mathcal{D}(\mathcal{A}_\mathrm{dir} (T))$. 
We define $X [Y_{(k)}/T_{k}] \in \mathcal{D}(\mathcal{A}_\mathrm{dir} (T))$ by defining its probability for each $\gamma\in \mathcal{A}_\mathrm{dir}(T)$ as follows. 
\begin{equation} \label{eq:alphayktk02}
X [Y_{(k)}/T_{k}] (\gamma) = 
\sum_{\alpha \in \mathcal{A}_\mathrm{dir} (T)} X(\alpha) \cdot \alpha [Y_{(k)}/T_{k}] (\gamma)
\end{equation}
\end{defi}

In other words, we have the following. 
\begin{equation}  \label{eq:alphayktk03}
X [Y_{(k)}/T_{k}] (\gamma) 
=
\sum_{\substack{\alpha \in \mathcal{A}_\mathrm{dir} (T) \\ \alpha \equiv \gamma ~ (\hat{k})}}
 X(\alpha) 
Y_{(k)} (\gamma_{k})
 = \mathrm{Pr}_{X}[X \equiv \gamma ~ (\hat{k})] \cdot Y_{(k)} (\gamma_{k})
\end{equation}
%

\begin{lemm} \label{lemm:prob_perm} 
For each permutation $\sigma \in \mathfrak{S}_{n}$, we investigate the event ``$\alpha$ is order $\sigma$'' for $\alpha \in \mathcal{A}_\mathrm{dir} (T)$. 
Then probability of this event measured by $X$ is the same as that measured by $X [Y_{(k)}/T_{k}]$.  
\end{lemm}

\begin{proof} 
It is straightforward. 
\end{proof} 

\begin{defi} \label{defi:rep_subalgo03} 
Suppose $\mathbf{t}=t_{1}\cdots t_{n} \in \{ 0,1 \}^{n}$. 
\begin{enumerate}
\item 
We call an assignment $\omega \in \Omega(T)$ is \emph{$\mathbf{t}$-type} if $v_{1}, \dots, v_{n}$ have values $t_{1}, \dots, t_{n}$ in the presence of $\omega$. 
\item 
For each $\sigma \in \mathfrak{S}_{n}$ and $k \in \{ 1,\dots,n \}$, we define an event $E^{k,\mathbf{t}} (\sigma)$ as follows. 

$E^{k, \mathbf{t}} (\sigma)$ : ~ 
There exists $\alpha \in \mathcal{A}_\mathrm{dir} (T)$ and $\omega \in \Omega(T)$ such that $\alpha$ is order $\sigma$, $\omega$ is $\mathbf{t}$-type, and $\alpha$ probes $T_{k}$ in the presence of $\omega$. 
\end{enumerate}
\end{defi}

$E^{k,\mathbf{t}} (\sigma)$ is equivalent to the following: 
For each $\alpha \in \mathcal{A}_\mathrm{dir} (T)$ and for each $\omega \in \Omega(T)$, if $\alpha$ is order $\sigma$ and $\omega$ is $\mathbf{t}$-type then it holds that $\alpha$ probes $T_{k}$ in the presence of $\omega$. 
Probability of $E^{k, \mathbf{t}}$ measured by $X$ is given as follows. 
\begin{equation} \label{eq:defofe1}
\mathrm{Pr}_{X}[E^{k, \mathbf{t}} ] 
= \sum_{\substack{\alpha \in \mathcal{A}_\mathrm{dir} (T) \\ E^{k, \mathbf{t}} (\sigma_{\alpha})}} X(\alpha)
\end{equation}

\begin{lemm}  \label{lemm:rep_subalgox} 
Let $X \in \mathcal{D}(\mathcal{A}_\mathrm{dir}(T))$ and $\omega$ an assignment of $\mathbf{t}$-type. 
Then we have the following. 
\begin{equation} \label{eq:rep_subalgo01a}
\mathrm{cost}(X,\omega) = \sum_{j=1}^{n} \sum_{\alpha_{(j)} \in \mathcal{A}_\mathrm{dir} (T_{j})} 
\left[
\left(
\sum_{\substack{\alpha \in  \mathcal{A}_\mathrm{dir} (T) :\\ \alpha_{j} = \alpha_{(j)} \\ \text{ and } E^{j, \mathbf{t}}(\sigma_{\alpha} )  }}
X(\alpha)
\right)
\times 
\mathrm{cost} (\alpha_{(j)}, \omega_{j})
\right]
\end{equation}
\end{lemm}

\begin{proof}
For each $j \in \{ 1,\dots,n \}$ and $\alpha \in \mathcal{A}_\mathrm{dir} (T)$, let symbol 
$[E^{j,\mathbf{t}}(\sigma_{\alpha})]$ denote the following. 
\begin{equation} \label{eq:rep_subalgo01b}
[E^{j,\mathbf{t}}(\sigma_{\alpha})] = 
\begin{cases}
1 & \text{ if $E^{j,\mathbf{t}}(\sigma_{\alpha})$ holds }
\\
0 & \text{ otherwise}
\end{cases}
\end{equation}
Then we have the following. 
\begin{align}
\mathrm{cost}(X,\omega) 
=&
\sum_{\alpha \in \mathcal{A}_\mathrm{dir} (T)} X(\alpha) \sum_{j=1}^{n} [E^{j,\mathbf{t}}(\sigma_{\alpha})] \mathrm{cost} (\alpha_{j}, \omega_{j})
\notag
\\
=&
\sum_{j=1}^{n} \sum_{\alpha \in \mathcal{A}_\mathrm{dir} (T)} X(\alpha) [E^{j,\mathbf{t}}(\sigma_{\alpha})] \mathrm{cost} (\alpha_{j}, \omega_{j})
\notag
\\
=&
\sum_{j=1}^{n} 
\sum_{\alpha_{(j)} \in \mathcal{A}_\mathrm{dir} (T_{j})} 
\sum_{\substack{\alpha \in \mathcal{A}_\mathrm{dir} (T) : \\ \alpha_{j} = \alpha_{(j)}}} 
X(\alpha) [E^{j,\mathbf{t}}(\sigma_{\alpha})] \mathrm{cost} (\alpha_{j}, \omega_{j})
\notag
\\
=&
\sum_{j=1}^{n} 
\sum_{\alpha_{(j)} \in \mathcal{A}_\mathrm{dir} (T_{j})} 
\sum_{\substack{\alpha \in \mathcal{A}_\mathrm{dir} (T) : \\ \alpha_{j} = \alpha_{(j)}}} 
X(\alpha) [E^{j,\mathbf{t}}(\sigma_{\alpha})] \mathrm{cost} (\alpha_{(j)}, \omega_{j})
\notag
\\
=&
\sum_{j=1}^{n} 
\sum_{\alpha_{(j)} \in \mathcal{A}_\mathrm{dir} (T_{j})} 
\left[
\left(
\sum_{\substack{\alpha \in \mathcal{A}_\mathrm{dir} (T) : \\ \alpha_{j} = \alpha_{(j)}}} 
X(\alpha) [E^{j,\mathbf{t}}(\sigma_{\alpha})] 
\right)
\times
\mathrm{cost} (\alpha_{(j)}, \omega_{j})
\right]
\label{eq:rep_subalgo01c}
\end{align}

The right-most formula of \eqref{eq:rep_subalgo01c} equals the right-hand side of \eqref{eq:rep_subalgo01a}. 
\end{proof}

\begin{lemm}  \label{lemm:rep_subalgoy} 
Let $X \in \mathcal{D}(\mathcal{A}_\mathrm{dir}(T))$ and $\omega$ an assignment of $\mathbf{t}$-type.
Suppose that $j,k$ are distinct members of $\{ 1,\dots, n \}$, $\alpha_{(j)} \in  \mathcal{A}_\mathrm{dir} (T_{j})$, $Y_{(k)}$ is an RDA on $T_{k}$, and that $Y=X[Y_{(k)}/T_{k}]$. 
Then we have the following. 
\begin{equation} \label{eq:rep_subalgo01d}
\sum_{\substack{\alpha \in  \mathcal{A}_\mathrm{dir} (T) :\\ \alpha_{j} = \alpha_{(j)} \\ \text{ and } E^{j, \mathbf{t}}(\sigma_{\alpha} )  }}
Y(\alpha)
=
\sum_{\substack{\alpha \in  \mathcal{A}_\mathrm{dir} (T) :\\ \alpha_{j} = \alpha_{(j)} \\ \text{ and } E^{j, \mathbf{t}}(\sigma_{\alpha} )  }}
X(\alpha)
\end{equation}
\end{lemm}

\begin{proof} 
It is sufficient to see the case where $j=2$ and $k=1$. Note that the bound variable $\alpha$ in the left-hand side of \eqref {eq:rep_subalgo01d} may be replaced by other letter, say $\gamma$. 
We have the following. 
\begin{align}
\sum_{\substack{\gamma \in  \mathcal{A}_\mathrm{dir} (T) :\\ \gamma_{2} = \alpha_{(2)} \\ \text{ and } E^{2, \mathbf{t}}(\sigma_{\gamma} )  }}
Y(\gamma)
=& 
\sum_{\substack{\gamma \in  \mathcal{A}_\mathrm{dir} (T) :\\ \gamma_{2} = \alpha_{(2)} \\ \text{ and } E^{2, \mathbf{t}}(\sigma_{\gamma} )  }}
\sum_{\substack{\alpha \in \mathcal{A}_\mathrm{dir} (T) \\ \alpha \equiv \gamma ~ (\hat{1})}}
 X(\alpha) 
Y_{(1)} (\gamma_{1}) \quad \text{[by \eqref{eq:alphayktk03}]}
\notag
\\
=&
\sum_{(\alpha, \gamma)}
 X(\alpha) 
Y_{(1)} (\gamma_{1}) 
\label{eq:rep_subalgo01e}
\end{align}
In the right-most formula of \eqref{eq:rep_subalgo01e}, $(\alpha, \gamma) \in \mathcal{A}_\mathrm{dir} (T)^{2}$ runs over the ones satisfying the following. 
\begin{equation} \label{eq:rep_subalgo01f}
\gamma_{2} = \alpha_{(2)} \land E^{2, \mathbf{t}}(\sigma_{\gamma} ) \land \alpha \equiv \gamma (\hat{1})
\end{equation}
By the definition of $\alpha \equiv \gamma (\hat{1})$ (see \eqref{eq:equivhat01}), 
the condition \eqref{eq:rep_subalgo01f} is equivalent to the following. 
\begin{equation} \label{eq:rep_subalgo01g}
\alpha_{2} = \alpha_{(2)} \land E^{2, \mathbf{t}}(\sigma_{\alpha} ) \land \gamma \equiv \alpha (\hat{1})
\end{equation}
Therefore, the right-most formula of \eqref{eq:rep_subalgo01e} can be transformed as follows. 
\begin{align}
\sum_{(\alpha, \gamma)}
 X(\alpha) 
Y_{(1)} (\gamma_{1}) 
=& \sum_{\substack{\alpha \in  \mathcal{A}_\mathrm{dir} (T) :\\ \alpha_{2} = \alpha_{(2)} \\ \text{ and } E^{2, \mathbf{t}}(\sigma_{\alpha} )  }}
\sum_{\substack{\gamma \in \mathcal{A}_\mathrm{dir} (T) \\ \gamma \equiv \alpha ~ (\hat{1})}}
X(\alpha) 
Y_{(1)} (\gamma_{1}) 
\notag
\\
=&
\sum_{\substack{\alpha \in  \mathcal{A}_\mathrm{dir} (T) :\\ \alpha_{2} = \alpha_{(2)} \\ \text{ and } E^{2, \mathbf{t}}(\sigma_{\alpha} )  }}
X(\alpha) 
\sum_{\substack{\gamma \in \mathcal{A}_\mathrm{dir} (T) \\ \gamma \equiv \alpha ~ (\hat{1})}}
Y_{(1)} (\gamma_{1}) 
\label{eq:rep_subalgo01h}
\end{align}
In the inner summation $\sum_{\substack{\gamma \in \mathcal{A}_\mathrm{dir} (T) \\ \gamma \equiv \alpha ~ (\hat{1})}}
Y_{(1)} (\gamma_{1}) $ of \eqref{eq:rep_subalgo01h}, $\gamma \in \mathcal{A}_\mathrm{dir} (T)$ runs over ones satisfying the following. 
\begin{equation} \label{eq:rep_subalgo01i}
\sigma_{\gamma} = \sigma_{\alpha} \land 
\gamma_{j} = \alpha_{j} \text{ for each $j \in \{ 2, \dots, n \}$} \land \gamma_{1} \in  \mathcal{A}_\mathrm{dir} (T_{1})
\end{equation}
Therefore, the inner summation is over $\gamma_{1} \in  \mathcal{A}_\mathrm{dir} (T_{1})
$. Thus we have the following. 
\begin{equation}
\sum_{\substack{\gamma \in \mathcal{A}_\mathrm{dir} (T) \\ \gamma \equiv \alpha ~ (\hat{1})}}
Y_{(1)} (\gamma_{1})  
=
\sum_{\gamma_{1} \in \mathcal{A}_\mathrm{dir} (T_{1}) }
Y_{(1)} (\gamma_{1})  
=
1
\label{eq:rep_subalgo01j}
\end{equation}
By \eqref{eq:rep_subalgo01h} and \eqref{eq:rep_subalgo01j}, we get \eqref{eq:rep_subalgo01d}. 
\end{proof}

\subsection{Equilibrium of directional algorithms} \label{subsection:rdavsrdir}

For each $k \in \{ 1, \dots, n \}$, by Lemma~\ref{lemm:sw86lem3_1}, there is an RDA $B_{k}$ on $T_{k}$ that achieves both $d_{0}(T)$ and $d_{1}(T)$. We fix such $B_{k}$ for each $k$. 

\begin{theo} \label{theo:rdiri_is_di} 
\begin{enumerate}
\item
Let $X \in \mathcal{D}(\mathcal{A}_\mathrm{dir} (T))$ and $k\in\{1,\ldots,n\}$. Let $B_k\in\mathcal{D}(\mathcal{A}_{\mathrm{dir}}(T_k))$ as above and let $Y = X[B_{k}/T_{k}]$. Then we have the following.
\begin{align}
\max_{\xi \in \Omega(T)} \mathrm{cost}(Y, \xi) \leq \max_{\xi \in \Omega(T)} \mathrm{cost}(X, \xi) 
\label{eq:rdir0_is_d0a01}
\\
\max_{\xi \in \Omega_0(T)} \mathrm{cost}(Y, \xi) \leq \max_{\xi \in \Omega_0(T)} \mathrm{cost}(X, \xi) 
\label{eq:rdir0_is_d0a02}
\\
\max_{\xi \in \Omega_1(T)} \mathrm{cost}(Y, \xi) \leq \max_{\xi \in \Omega_1(T)} \mathrm{cost}(X, \xi) 
\label{eq:rdir0_is_d0a03}
\end{align}
\item
Let $X \in \mathcal{D}(\mathcal{A}_\mathrm{dir} (T))$. Then there exists an RDA $Z$ on $T$ with the following properties.
\begin{align}
\max_{\xi \in \Omega(T)} \mathrm{cost}(Z, \xi) \leq \max_{\xi \in \Omega(T)} \mathrm{cost}(X, \xi) 
\label{eq:rdiri_is_di01}
\\
\max_{\xi \in \Omega_0(T)} \mathrm{cost}(Z, \xi) \leq \max_{\xi \in \Omega_0(T)} \mathrm{cost}(X, \xi) 
\label{eq:rdiri_is_di02}
\\
\max_{\xi \in \Omega_1(T)} \mathrm{cost}(Z, \xi) \leq \max_{\xi \in \Omega_1(T)} \mathrm{cost}(X, \xi) 
\label{eq:rdiri_is_di03}
\end{align}
\item $R_{\mathrm{dir},i}(T) = d_{i}(T)$.
\end{enumerate}
\end{theo}

\begin{proof}
By induction on the height of $T$, we are going to show all three assertions simultaneously. 
If the height is 0 then the assertions are obvious. 
Assume that assertions (1), (2) and (3) hold for all trees whose height are less than that of $T$. 
By induction hypothesis on (3), $B_{k}$ achieves $R_{\mathrm{dir},0}(T_{k})$ and $R_{\mathrm{dir},1}(T_{k})$, too. 

For each $A\in\mathcal{D}(\mathcal{A}_{\mathrm{dir}}(T))$ and $\sigma \in \mathfrak{S}_{n}$, let $\pi_{A} (\sigma) $ denote the probability of $A$ being order $\sigma$, that is, 
\begin{equation} \label{eq:pix}
\pi_{A} (\sigma) = \sum_{\substack{\alpha \in \mathcal{A}_\mathrm{dir} (T) \\ \sigma_{\alpha} = \sigma}} A(\alpha).
\end{equation}

(1). We show \eqref{eq:rdir0_is_d0a02} in the case of $k=1$. 
The other cases of \eqref{eq:rdir0_is_d0a02} and the other equations \eqref{eq:rdir0_is_d0a01}, \eqref{eq:rdir0_is_d0a03} are shown in the same way. 
Let $\omega \in \Omega_0(T)$ be an assignment that maximizes $\mathrm{cost}(Y,\omega)$ under the constraint that the root has value 0. Assume that $\mathbf{t}\in \{ 0,1 \}^{n}$ and $\omega$ is $\mathbf{t}$-type. 

Case 1: Suppose we have $\mathrm{Pr}_X [ E^{1,\mathbf{t}} ] = 0$. Then we have the following. 
\begin{equation}
\max_{\xi\in\Omega_0(T)}\mathrm{cost}(Y,\xi) = \mathrm{cost}(Y,\omega) = \mathrm{cost}(X,\omega) \leq \max_{\xi\in\Omega_0(T)}\mathrm{cost}(X,\xi)
\end{equation}

Case 2: Supppose we have $\mathrm{Pr}_X [ E^{1,\mathbf{t}} ] > 0$. We define $A_{1} \in \mathcal{D}(\mathcal{A}_\mathrm{dir} (T_{1}))$ by defining its probability for each $\alpha_{(1)} \in \mathcal{A}_{\mathrm{dir}}(T_{1})$ as follows. 
\begin{equation} \label{eq:rdiri_is_di01b}
A_{1} (\alpha_{(1)}) = \frac{1}{\mathrm{Pr}_X [ E^{1,\mathbf{t}} ]} 
\sum_{\substack{\alpha \in  \mathcal{A}_\mathrm{dir} (T) :\\ \alpha_{1} = \alpha_{(1)} \\ \text{ and } E^{ 1, \mathbf{t}}(\sigma_{\alpha})  }}
X(\alpha)
\end{equation}

We may assume the following, where the suffix of the $\max$ is $\xi_{(1)}\in\Omega_{t_1}(T_1)$.
\begin{equation} \label{eq:rdiri_is_di02b}
\mathrm{cost}(B_{1}, \omega_{1}) = \max_{\xi_{(1)}\in\Omega_{t_1}(T_1)} \mathrm{cost}(B_{1},\xi_{(1)}).
\end{equation}

Now, let $\eta_{(1)}$ be an assignment to $T_{1}$ that sets the value of $v_{1}$ to $t_{1}$ and satisfies the following.
\begin{equation} \label{eq:rdiri_is_di03b}
\mathrm{cost}(A_{1}, \eta_{(1)}) = \max_{\xi_{(1)}\in\Omega_{t_1}(T_1)} \mathrm{cost}(A_{1},\xi_{(1)})
\end{equation}

Since $B_1$ achieves $R_{\mathrm{dir},t_1} (T_1)$, we have the following. 
\begin{equation} \label{eq:rdiri_is_di04}
\mathrm{cost}(B_{1}, \omega_{1}) \leq \mathrm{cost}(A_{1}, \eta_{(1)})
\end{equation}

Let $\omega^{\prime}$ be the assignment given by substituting $\eta_{(1)}$ for $\omega_{1}$ in $\omega$. 
Then $\omega^{\prime}$ is $\mathbf{t}$-type, too. 
By Lemmas~\ref{lemm:rep_subalgox}, \ref{lemm:rep_subalgoy}, and \eqref{eq:defofe1} we have the following. 
\begin{align} 
&\mathrm{cost}(X,\omega^{\prime}) - \mathrm{cost}(Y,\omega) 
\notag 
\\
=& 
\mathrm{Pr}_X [ E^{1,\mathbf{t}} ] \mathrm{cost}(A_{1}, \eta_{(1)}) 
-
\sum_{\alpha_{(1)} \in \mathcal{A}_\mathrm{dir} (T_{1})} 
\left(
\sum_{\substack{\alpha \in  \mathcal{A}_\mathrm{dir} (T) :\\ \alpha_{1} = \alpha_{(1)} \\ \text{ and } E^{k,\mathbf{t}} (\sigma_{\alpha})  }}
Y (\alpha)
\right)
\times 
\mathrm{cost} (\alpha_{(1)}, \omega_{1})
\label{eq:rep_subalgo02}
\end{align}

We look at the coefficient of $\mathrm{cost} (\alpha_{(1)}, \omega_{1})$ in \eqref{eq:rep_subalgo02}. 
Note that we may replace the bound variable $\alpha$ by other letter, say $\gamma$. We have

\begin{align}
& \sum_{\substack{\gamma \in  \mathcal{A}_\mathrm{dir} (T) :\\ \gamma_{1} = \alpha_{(1)} \\ \text{ and } E^{1,\mathbf{t}} (\sigma_{\gamma})  }}
Y (\gamma) 
\notag
\\
=&  \sum_{\substack{\gamma \in  \mathcal{A}_\mathrm{dir} (T) :\\ \gamma_{1} = \alpha_{(1)} \\ \text{ and } E^{1,\mathbf{t}} (\sigma_{\gamma})  }}
\mathrm{Pr}_{X} [X \equiv \gamma ~ (\hat{1})] \cdot B_{1} (\gamma_{1}) 
\quad \text{[By \eqref{eq:alphayktk03} and the definition of $Y$.]}
\notag
\\
=&  \sum_{\substack{\gamma \in  \mathcal{A}_\mathrm{dir} (T) :\\ \gamma_{1} = \alpha_{(1)} \\ \text{ and } E^{1,\mathbf{t}} (\sigma_{\gamma})  }}
\mathrm{Pr}_{X} [X \equiv \gamma ~ (\hat{1})] \cdot B_{1} (\alpha_{(1)}) 
\notag
\\
=& B_{1} (\alpha_{(1)}) \sum_{\substack{\gamma \in  \mathcal{A}_\mathrm{dir} (T) :\\ \gamma_{1} = \alpha_{(1)} \\ \text{ and } E^{1,\mathbf{t}} (\sigma_{\gamma})  }}
\mathrm{Pr}_{X} [X \equiv \gamma ~ (\hat{1})].
\label{eq:rep_subalgo03}
\end{align}

Then we look at the second factor of \eqref{eq:rep_subalgo03}. 
Under the constraint that $\gamma \in  \mathcal{A}_\mathrm{dir} (T) $ and $\gamma_{1} = \alpha_{(1)}$, 
tuple $(\sigma_{\gamma}, \gamma_{2}, \dots, \gamma_{n})$ can take arbitrary value in $\mathfrak{S}_{n} \times \mathcal{A}_\mathrm{dir} (T_{2}) \times \cdots \mathcal{A}_\mathrm{dir} (T_{n})$. 
Thus, we have the following. 

\begin{align}
& \sum_{\substack{\gamma \in  \mathcal{A}_\mathrm{dir} (T) :\\ \gamma_{1} = \alpha_{(1)} \\ \text{ and } E^{1,\mathbf{t}} (\sigma_{\gamma})  }}
\mathrm{Pr}_{X}[X \equiv \gamma ~ (\hat{1})] 
\notag
\\
=& \sum_{\sigma \in \mathfrak{S}_{n}: E^{1,\mathbf{t}} (\sigma)  } 
\sum_{\substack{\gamma_{(2)} \in \mathcal{A}_\mathrm{dir} (T_{2}) \\ \dots \\ \gamma_{(n)} \in \mathcal{A}_\mathrm{dir} (T_{n})}} 
\mathrm{Pr}_{X}[X \equiv \gamma ~ (\hat{1})] 
\label{eq:rep_subalgo04}
\\
& \text{[$\gamma$ in \eqref{eq:rep_subalgo04} denotes the unique $\gamma \in \mathcal{A}_\mathrm{dir} (T)$ that has order $\sigma$, }
\notag
\\
& ~ \text{$\gamma_{1} = \alpha_{(1)}$, and $\gamma_{j} = \gamma_{(j)}$ for each $j \in \{ 2, \dots, n \}$. ]}
\notag
\\
=& \sum_{\sigma \in \mathfrak{S}_{n}: E^{1,\mathbf{t}} (\sigma)  } 
\pi_X(\sigma)
\quad \text{[see \eqref{eq:pix} for the definition of $\pi_{X}$]}
\notag
\\
=&
\mathrm{Pr}_{X}[E^{1,\mathbf{t}}]
\label{eq:rep_subalgo06}
\end{align}

By \eqref{eq:rep_subalgo03} and \eqref{eq:rep_subalgo06}, we get
\begin{equation} \label{eq:rep_subalgo07}
\sum_{\substack{\gamma \in  \mathcal{A}_\mathrm{dir} (T) :\\ \gamma_{1} = \alpha_{(1)} \\ \text{ and } E^{1,\mathbf{t}} (\sigma_{\gamma})  }}
Y (\gamma) 
= B_{1} (\alpha_{(1)}) \mathrm{Pr}_{X}[E^{1,\mathbf{t}}],
\end{equation}

therefore
\begin{align} 
&\sum_{\alpha_{(1)} \in \mathcal{A}_\mathrm{dir} (T_{1}) } 
\left(
\sum_{\substack{\alpha \in  \mathcal{A}_\mathrm{dir} (T) :\\ \alpha_{1} = \alpha_{(1)} \\ \text{ and } E^{k,\mathbf{t}} (\sigma_{\alpha})  }}
Y (\alpha)
\right)
\times 
\mathrm{cost} (\alpha_{(1)}, \omega_{1})
\notag
\\
=&
\mathrm{Pr}_{X}[E^{1,\mathbf{t}}]
\sum_{\alpha_{(1)} \in \mathcal{A}_\mathrm{dir} (T_{1})} 
B_{1} (\alpha_{(1)}) \mathrm{cost} (\alpha_{(1)}, \omega_{1})
\notag
\\
=&
\mathrm{Pr}_{X}[E^{1,\mathbf{t}}] ~ 
\mathrm{cost} (B_{(1)}, \omega_{1}).
\label{eq:rep_subalgo08}
\end{align}

By \eqref{eq:rdiri_is_di04}, \eqref{eq:rep_subalgo02}, and\eqref{eq:rep_subalgo08}, we have
\begin{align}
&\mathrm{cost}(X,\omega^{\prime}) - \mathrm{cost}(Y,\omega) 
\notag
\\
=&  \mathrm{Pr}_{X}[E^{1,\mathbf{t}}] (\mathrm{cost}(A_{1}, \eta_{(1)}) - \mathrm{cost} (B_{1}, \omega_{1}))
\notag
\\
\geq & 0
\label{eq:rep_subalgo09}
\end{align}

and we can see that
\begin{align} 
\max_{\xi\in\Omega_0(T)} \mathrm{cost}(Y,\xi )
&=
\mathrm{cost}(Y, \omega) \quad \text{[by the definition of $\omega$]}
\notag \\
&\leq \mathrm{cost}(X, \omega^{\prime}) \quad \text{[by \eqref{eq:rep_subalgo09}]}
\notag \\
&\leq \max_{\xi\in\Omega_0(T)} \mathrm{cost}(X, \xi).\label{eq:rdiri_is_di05}
\end{align}

(2). We show induction step of \eqref{eq:rdiri_is_di02}. 
The other equations \eqref{eq:rdiri_is_di01}, \eqref{eq:rdiri_is_di03} are shown in the same way. 
By starting from $X$, we can repeatedly apply (1) for $k \in \{1,\dots,n\}$. 
We define $Y \in \mathcal{A}_\mathrm{dir}(T)$ by setting $Y:=( \cdots (X [B_{1}/T_{1}]) \cdots ) [B_{n}/T_{n}]$. 
Then we have the following, where (b) holds by Lemma~\ref{lemm:prob_perm}, and (c) is shown in the same way as \eqref{eq:rep_subalgo08}.  
\begin{enumerate}
\item[(a)]
$
\displaystyle \max_{\xi \in \Omega_0(T)} \mathrm{cost}(Y, \xi) \leq \max_{\xi \in \Omega_0(T)} \mathrm{cost}(X, \xi).
$
\item [(b)] $\pi_{Y} = \pi_{X}$. 
\item [(c)] For each $j \in \{1,\dots,n\}$, $\mathbf{t} \in \{ 0, 1 \}^{n}$ and $\omega \in \Omega(T)$ of $\mathbf{t}$-type,
\begin{equation} \label{eq:cost_y_omega}
\mathrm{cost}(Y,\omega) = \sum_{j=1}^{n} \mathrm{Pr}_{Y} [E^{j,\mathbf{t}}] ~ \mathrm{cost}(B_{j}, \omega_{j}).
\end{equation}
\end{enumerate}

Let $Z$ be the RDA such that $\pi_{Z} = \pi_{X}$ and for each $j \in \{1,\dots,n\}$, whenever $Z$ evaluates $T_{j}$, $Z$ behaves as $B_{j}$. Then for each $j \in \{1,\dots,n\}$, $\mathbf{t} \in \{ 0, 1 \}^{n}$, and $\omega \in \Omega(T)$ of $\mathbf{t}$-type, the following holds. 
\begin{align} 
\mathrm{cost}(Z,\omega) 
&= \sum_{j=1}^{n} \mathrm{Pr}_{Z} [E^{j,\mathbf{t}}] ~ \mathrm{cost}(B_{j}, \omega_{j}) 
\quad \text{[This is easily shown.]}
\notag
\\
&= \sum_{j=1}^{n} \mathrm{Pr}_{Y} [E^{j,\mathbf{t}}] ~ \mathrm{cost}(B_{j}, \omega_{j}) 
\quad \text{[by $\pi_{Z}=\pi_{Y}$]}
\notag
\\
&=\mathrm{cost}(Y,\omega)
\quad \text{[by \eqref{eq:cost_y_omega}]}
\label{eq:cost_z_omega}
\end{align}

By (a) and \eqref{eq:cost_z_omega}, we get \eqref{eq:rdiri_is_di02}. 
(Note that We do not have to show that $Z$ and $Y$ are identical as Boolean decision trees. For our purpose, (a) and \eqref{eq:cost_z_omega} are sufficient.) 

(3). Follows from (2).
\end{proof}

\section{Randomized depth-first algorithms} \label{sec:problem2}

In this section, we are going to prove that $R_{\mathrm{DF}, i} (T) = R_{\mathrm{dir}, i} (T)$ for $i=0,1$.
Again, we assume that the root of tree $T$ has $n$ child nodes $v_{1}, \dots, v_{n}$, and the corresponding subtrees are $T_{1},\dots,T_{n}$. For $\mathbf{t} = t_1 \cdots t_n \in \{ 0,1 \}^n $, we let expression ``$\bar{v}$ is $\mathbf{t}$'' denote ``$v_{1}, \dots, v_{n}$ have values $t_{1}, \dots, t_{n}$ respectively.''

\subsection{Outline of the proof in the remaining} \label{subsec:4outline}

Roughly speaking, the goal and the outline of the proof in this section is as follows. 

(1) Recall von Neumann-type equilibria in Definition~\ref{defi:vonneumanneq}. 
Our immediate goal is their equality $\overline{P}_{\mathrm{DF},i} (T) = \overline{P}_{\mathrm{dir},i} (T)$. 
The final goal $R_{\mathrm{DF}, i} (T) = R_{\mathrm{dir}, i} (T)$ will follow from the equality above. 

(2) Without loss of generality, we may assume that the root of $T$ is labeled with AND. 
The proof of $\overline{P}_{\mathrm{DF},i} (T) = \overline{P}_{\mathrm{dir},i} (T)$ for $i=0$ is more difficult than that for $i=1$. 

(3) As to the proof for $i=0$, our ideal scenario would be this: For any distribution $S$ on $\Omega_{0}(T)$ and for any $\alpha \in \mathcal{A}_{\mathrm{DF}} (T)$, there exists  $B \in \mathcal{D}(\mathcal{A}_{\mathrm{dir}} (T))$ such that $\mathrm{cost}(\alpha, S) \geq \mathrm{cost}(B,S)$. 
In general, it is hard to find such $B$. 

(4) By cleverly devising a method, we will get a result very close to the ideal scenario mentioned above. Arrows in Figure~\ref{fig:dependency01} illustrates interdependence among definitions in this section. 

\begin{figure}[h] 
  \centering
  \includegraphics[width=110mm,bb=0 0 1422 596]{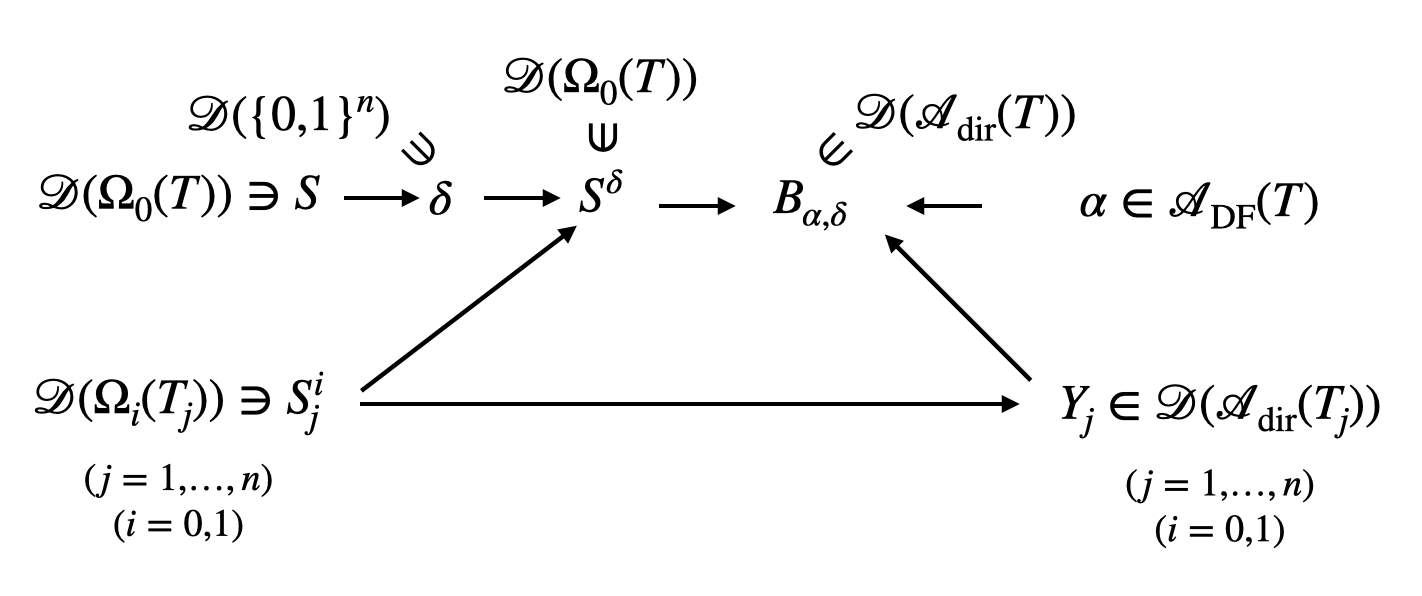}
  \caption{Interdependence of definitions}
  \label{fig:dependency01}
\end{figure}

(5) Probability distribution $S^{\delta}$ in Figure~\ref{fig:dependency01} will be defined as follows. 
We fix a probability distribution $S \in \mathcal{D} (\Omega_{0} (T))$ with a certain good property. 
We extract $\delta \in \mathcal{D} (\{ 0,1 \}^{n})$ from $S$. 
Probability distribution $\delta$ describes the behavior of $S$ at the depth-1 level. 
For each subtree $T_{j}$ and $i=0,1$, we take probability distributions $S^{i}_{j} \in \mathcal{D} (\Omega_{i} (T_{j}))$ of a certain good nature, where $i$ denotes the value of $v_{j}$, the roof of $T_{j}$. By means of $\delta$ and $\{ S^{i}_{j} \}_{i,j}$, we define $S^{\delta}$. At the depth-1 level of $T$, $S^{\delta}$ behaves $S$-like. On each $T_{j}$, $S^{\delta}$ behaves $S^{i}_{j} $-like. 

(6) The ``chimera'' randomized directional algorithm $B_{\alpha,\delta}$ in Figure~\ref{fig:dependency01} will be defined as follows. 
By induction hypothesis on the height of trees, we have $\overline{P}_{\mathrm{DF},i} (T_{j}) = \overline{P}_{\mathrm{dir},i} (T_{j})$, and we can take a randomized directional algorithm $Y_{j}$ on $T_{j}$ that behaves well both for $S^{0}_{j}$ and $S^{1}_{j}$. A depth-first algorithm $\alpha$ on $T$ is given, where $\alpha$ may not be directional. Randomized directional algorithm $B_{\alpha,\delta}$ on $T$ is defined by means of $\alpha, Y_{j}$, and $\delta$. At the depth-1 level of $T$, $B_{\alpha,\delta}$ behaves similarly to the movement of $\alpha$ with respect to $S^{\delta}$. On each $T_{j}$, $B_{\alpha,\delta}$ behaves $Y_{j}$-like. 

(7) We will show $\mathrm{cost}(\alpha, S^{\delta}) \geq \mathrm{cost}(B_{\alpha,\delta},S^{\delta})$ (see \eqref{eq:lboundalphasdelta2}), which is a key to the proof of the main result. 

In subsection~\ref{subsection:tools}, we define $S^{\delta}$ and investigate its behavior with respect to directional algorithms. 
In subsection~\ref{subsection:main}, we construct $B_{\alpha,\delta}$ and show the main result.

\subsection{Mixed strategy for independent distributions} \label{subsection:tools}

In this subsection, we define a certain mixed strategy for independent distributions (Definition~\ref{defi:sdeltadir}). We are going to show that this type of mixed strategy is very effective against directional algorithms (Lemma~\ref{lemm:sdelta_det}). 

\begin{defi} \label{defi:sdeltadir} 
Suppose that $\delta \in \mathcal{D}(\{ 0,1 \}^{n})$, in other words $\delta$ is a probability distribution on $\{0,1\}^n$, 
and that $S_{j}^{i} \in \mathcal{D}(\Omega_{i}(T_{j}))$ for each $j \in \{ 1,\dots, n\}$ and $i \in \{ 0,1 \}$. 
We let $\langle \delta; \{ S_{j}^{i} \}_{j,i} \rangle$ denote the following probability distribution on $\Omega(T)$, 
where $\omega_{j} \in \Omega(T_{j})$ for each $j$.
\begin{equation}
\langle \delta; \{ S_{j}^{i} \}_{j,i} \rangle (\omega_{1} \cdots \omega_{n}) 
= \sum_{\mathbf{t} \in \{ 0,1 \}^{n}} \delta (\mathbf{t}) \prod_{j=1}^{n} S_{j}^{t_{j}} (\omega_{j})
\end{equation} 
\end{defi}

\begin{lemm} \label{lemm:prsdelta_deltau}
Suppose that $\delta$ and $\{ S_{j}^{i} \}_{j,i} $ are as above. Let $S^{\delta}$ denote $\langle \delta; \{ S_{j}^{i} \}_{j,i} \rangle$. 
\begin{enumerate}
\item
$S^{\delta} (\omega_{1} \cdots \omega_{n}) = \delta (\mathbf{t}^{\omega}) \prod_{j=1}^{n} S_{j}^{t_{j}^{\omega}} (\omega_{j})$, where $\mathbf{t}^{\omega} \in \{ 0,1 \}^{n}$ is the string such that $t_{j}^{\omega}$ is the value of $v_{j}$ when $\omega_{j}$ is assigned to $T_{j}$.  
\item
$\sum_{\omega \in \Omega(T) } S^{\delta} (\omega) =1 $. Thus $S^{\delta}$ is indeed probability distribution. 
\item
For each $\mathbf{t} \in \{0,1 \}^{n}$, we have 
$\mathrm{Pr}_{S^{\delta}} [\text{$\bar{v}$ is $\mathbf{t}$}] = \delta (\mathbf{t})$.
\end{enumerate}
\end{lemm}

The proof of Lemma \ref{lemm:prsdelta_deltau} is routine, which will be given in Appendix.

\begin{lemm} \label{lemm:sdelta_det}
Let $T$ be an AND--OR tree with nonzero height. Let $v_1,\ldots,v_n$ be the depth-1 nodes of $T$, and $T_1,\ldots,T_n$ the corresponding subtrees. For each $S \in \mathcal{D}(\Omega(T))$, set $\delta_S \in \mathcal{D}( \{ 0,1 \}^{n} )$ as
\begin{equation} \label{eq:prszerov_det}
\delta_S (t_{1} \cdots t_{n}) = \mathrm{Pr}_{S} [ \text{$\bar{v}$ is $\mathbf{t}$} ] .
\end{equation}
Then there exists a series of distributions $\{S^i_j\}_{i\in\{0,1\},j\in\{1,\ldots,n\}}$ which satisfies the following.

(1) For each $i$ and $j$, we have $S^i_j \in \mathcal{D}(\Omega_i(T_j))$ and it achieves $P_{\mathrm{dir},i}(T_j)$.

(2) For any $S \in \mathcal{D}(\Omega(T))$, we have

\begin{equation} \label{eq:costzdeltabars0_det}
\min_{\alpha \in \mathcal{A}_\mathrm{dir}(T)} \mathrm{cost} (\alpha, S)
\leq 
\min_{\alpha \in \mathcal{A}_\mathrm{dir}(T)} \mathrm{cost} (\alpha, S ^{\delta}),
\end{equation}

where $S^{\delta}$ denotes $\langle \delta_S; \{ S_{j}^{i} \}_{j,i} \rangle$. 

(3) There exist $S^0 \in \mathcal{D}(\Omega_0 (T))$, $S^1 \in \mathcal{D}(\Omega_1 (T))$ and $\tilde{\alpha} \in \mathcal{A}_{\mathrm{dir}}(T)$ where
\begin{itemize}
\item $S^0 = \langle \delta_{0}; \{S^i_j\}_{j,i} \rangle$ for some $\delta_{0} \in\mathcal{D}(\{0,1\}^n)$ and it achieves $P_{\mathrm{dir},0}(T)$, 
\item $S^1 = \langle \delta_{1}; \{S^i_j\}_{j,i} \rangle $ for some $\delta_{1} \in\mathcal{D}(\{0,1\}^n)$ and it achieves $P_{\mathrm{dir},1}(T)$, and
\item $\tilde{\alpha}$ minimizes the costs of both $S^0$ and $S^1$, that is,
\begin{equation} \label{eq:defminimizer0}
\mathrm{cost} (\tilde{\alpha},S^0) = \min_{\alpha \in \mathcal{A}_{\mathrm{dir}} (T)} \mathrm{cost} (\alpha,S^0)
\end{equation}
and
\begin{equation} \label{eq:defminimizer1}
\mathrm{cost} (\tilde{\alpha},S^1) = \min_{\alpha \in \mathcal{A}_{\mathrm{dir}} (T)} \mathrm{cost} (\alpha,S^1)
\end{equation}
holds.
\end{itemize}
\end{lemm}

\begin{proof}
Our proof consists of a good amount of work. Proofs of some equations are elementary but not short. To avoid sidetracking discussion, some of them will be given in Appendix.

We show by induction on the height of $T$. We only see the case where the root node of $T$ is an AND node. (The OR case is similar.) It is straightforward if the height of $T$ is 1. For each $j\in\{1,\ldots,n\}$, we set $S^0_j \in \mathcal{D}( \Omega_0 (T_j) )$ and $S^1_j \in \mathcal{D}( \Omega_1 (T_j) )$ as follows.

If $T_j$ is a leaf, then $S^0_j$ and $S^1_j$ are trivial. Suppose $T_j$ has height greater than 0. Let $T_{j,1},\ldots,T_{j,m}$ be the depth-1 subtrees of $T_j$. By induction hypothesis, there exists a series of distribution $\{S^i_{j,k}\}_{i\in\{0,1\},k\in\{1,\ldots,m\}}$ which satisfies (1), (2) and (3) regarding $T_j$. In particular, we set $S^0_j$ and $S^1_j$ as in (3) of $T_j$. Note that if $T_j$ has OR root, we have

\begin{equation} \label{eq:subtree0dist}
S^{0}_{j} (\xi) := \prod_{k=1}^m S^0_{j,k} (\xi_{k}), 
\quad
S^{1}_{j} := \langle \delta_{j,1} ; \{ S_{j,k}^{i} \}_{k,i} \rangle ~ \text{ (for some $\delta_{j,1}$)}, 
\end{equation}

\begin{equation} \label{eq:subtree1dist}
S^{0}_{j} := \langle \delta_{j,0} ; \{ S_{j,k}^{i} \}_{k,i} \rangle ~ \text{ (for some $\delta_{j,0}$)}, 
\quad  
S^1_j (\xi) := \prod_{k=1}^m S^1_{j,k} (\xi_{k}).
\end{equation}

(1) clearly holds by definition of $S^0_j$ and $S^1_j$.

(2). Let $\sigma$ be a permutation on $\{1,\ldots,n\}$. If a directional algorithm $\alpha \in \mathcal{A}_{\mathrm{dir}}(T)$ evaluates the depth-1 subtrees of $T$ in the order of $T_{\sigma(1)}, \ldots, T_{\sigma(n)}$, then we say $\alpha$ is an {\it order $\sigma$} algorithm. We show the following.

\ 

{\bf Claim 1.} For each permutation $\sigma \in \mathfrak{S}_n$, we have
\begin{equation} \label{eq:mincostineq_sigma}
\min_{\alpha: \text{ order } \sigma} \mathrm{cost} (\alpha, S) \leq \min_{\alpha: \text{ order } \sigma} \mathrm{cost} (\alpha, S^\delta) ,
\end{equation}

\ 

(Proof of Claim 1.) Let $\alpha$ be a directional algorithm of order $\sigma$. Without loss of generality, we can assume that $\sigma = (1,2,\ldots,n)$, that is, $\alpha$ evaluates the subtrees $T_1,\ldots,T_n$ in numerical order. For each $j\in\{1,\ldots,n\}$, let $\alpha_j$ denote the $T_j$ part of $\alpha$. Note that since $\alpha$ is directional, each $\alpha_j$ does not depend on the query-answer history. We have
\begin{equation} \label{eq:generalcost_dir}
\begin{aligned}
& \mathrm{cost} (\alpha,S) \\
=& \sum_{j=1}^n \mathrm{Pr}_S [ \alpha \text{ evaluates } T_j] \\
& \sum_{\omega_j \in \Omega(T_j) } \mathrm{Pr}_S [ \omega_j \text{ is assigned to } T_j | \alpha \text{ evaluates } T_j ] \mathrm{cost} (\alpha_j,\omega_j) \\
=& \sum_{j=1}^n \mathrm{Pr}_S [ \alpha \text{ evaluates } T_j] \mathrm{cost} (\alpha_j,S_j),
\end{aligned}
\end{equation}
where $S_j\in\mathcal{D}(\Omega(T_j))$ is defined as
\begin{equation} \label{eq:defpartialdist}
S_j(\xi) = \mathrm{Pr}_S [ \xi \text{ is assigned to } T_j| \alpha \text{ evaluates } T_j ]
\end{equation}
for each $j\in\{1,\ldots,n\}$. Similarly, we have
\begin{equation} \label{eq:deltacost_dir}
\mathrm{cost} (\alpha,S^\delta) = \sum_{j=1}^n \mathrm{Pr}_{S^\delta} [ \alpha \text{ evaluates } T_j] \mathrm{cost} (\alpha_j,S^\delta_j),
\end{equation}
where $S^\delta_j\in\mathcal{D}(\Omega(T_j))$ is defined as
\begin{equation} \label{eq:defpartialdistb}
S^\delta_j(\xi) = \mathrm{Pr}_{S^\delta} [ \xi \text{ is assigned to } T_j| \alpha \text{ evaluates } T_j ]
\end{equation}

\noindent
for each $j\in\{1,\ldots,n\}$. Note that $S_j$ [resp. $S^\delta_j$] is determined by $S$ [resp. $S^\delta$] and $\sigma$, but not by specific procedure of $\alpha$ in each subtrees. By definition of $S^\delta$ 
and Lemma~\ref{lemm:prsdelta_deltau}, we have

\begin{equation} \label{eq:evalprob}
\mathrm{Pr}_S [ \alpha \text{ evaluates } T_j] = \mathrm{Pr}_{S^\delta} [ \alpha \text{ evaluates } T_j]
\end{equation}

\noindent
for each $j\in\{1,\ldots,n\}$. We denote this value by $p_j$. Similarly, we have

\begin{equation} \label{eq:zeroprob}
\mathrm{Pr}_S [v_j \text{ is } 0 \ | \ \alpha \text{ evaluates } T_j] = \mathrm{Pr}_{S^\delta} [v_j \text{ is } 0 \ | \ \alpha \text{ evaluates } T_j],
\end{equation}

\noindent
for each $j\in\{1,\ldots,n\}$. We denote this value by $q_j$. For each $j$, we set $U^0_j \in \mathcal{D}(\Omega_0(T_j))$ and $U^1_j \in \mathcal{D}(\Omega_1(T_j))$ as

\begin{equation} \label{eq:defpartialdist_root0}
U^0_j (\xi) = \mathrm{Pr}_S [\xi \text{ is assigned to } T_j \ | \ \alpha \text{ evaluates } T_j \text{ and } v_j \text{ is } 0]
\end{equation}

\noindent
and

\begin{equation} \label{eq:defpartialdist_root1}
U^1_j (\xi) = \mathrm{Pr}_S [\xi \text{ is assigned to } T_j \ | \ \alpha \text{ evaluates } T_j \text{ and } v_j \text{ is } 1] .
\end{equation}

\noindent
We now have the following. See Appendix for a detailed proof.
\begin{equation} \label{eq:partialdistcomp}
S_j = q_j U^0_j + (1-q_j) U^1_j, \quad S^\delta_j = q_j S^0_j + (1-q_j) S^1_j.
\end{equation}

\noindent
We now show the following, which is the key fact to prove Claim 1.

\ 

{\bf Claim 2.} For each $j\in\{1,\ldots,n\}$, we have 
\begin{equation} \label{eq:partialcostineq} 
\displaystyle \min_{\alpha_j \in \mathcal{A}_{\mathrm{dir}}(T_{j})} \mathrm{cost} (\alpha_j,S_j) \leq \min_{\alpha_j \in \mathcal{A}_{\mathrm{dir}}(T_{j})} \mathrm{cost} (\alpha_j,S^\delta_j).
\end{equation}

\ 

(Proof of Claim 2.) We only show the case where the root node of $T_j$ is an OR node. (The AND case is similar). It is clear if $T_j$ is a leaf, so we assume that the height of $T_j$ is nonzero. By induction hypothesis of (2), the distribution $S' := \langle \delta_{S_j} ; \{ S^i_{j,k} \}_{k,i} \rangle$ satisfies 

\begin{equation} \label{eq:subtreecostineq}
\min_{\alpha_j \in \mathcal{A}_{\mathrm{dir}} (T_j) } \mathrm{cost} (\alpha_j,S_j) \leq \min_{\alpha_j \in \mathcal{A}_{\mathrm{dir}} (T_j) } \mathrm{cost} (\alpha_j,S').
\end{equation}

Furthermore, since $T_j$ has OR root, for each $\xi \in \Omega_{0} (T_{j})$, we have $\mathbf{t}^{\xi} = 0^m$, where $m$ is the number of child nodes of $T_{j}$. 
By means of this fact, we can see that
\begin{equation} \label{eq:Sprimewritten}
S' = q_j S^0_j + (1-q_j) S'_1,
\end{equation}
where $S'_1 = \langle \delta_{U^1_j} ; \{ S^i_{j,k} \}_{k,i} \rangle$. See Appendix for a detailed proof of \eqref{eq:Sprimewritten}.
 
Now, let $\beta \in \mathcal{A}_{\mathrm{dir}} (T_j)$ be a minimizer of $S'_1$, that is, a directional algorithm which satisfies
\begin{equation} \label{eq:Sprime1minimizer}
\mathrm{cost} (\beta,S'_1) = \min_{\alpha_j \in \mathcal{A}_{\mathrm{dir}} (T_j) } \mathrm{cost} (\alpha_j,S'_1).
\end{equation} 
Without loss of generality, we can assume that $\beta$ evaluates the subtrees $T_{j,1},\ldots,T_{j,m}$ in numerical order. Then we have
\begin{equation} \label{eq:betacostSprime1}
\begin{aligned}
& \mathrm{cost} (\beta,S'_1) \\
=& \sum_{\mathbf{t}' \in \{0,1\}^m } \delta_{U^1_j} (\mathbf{t}') \sum_{k=1}^m (1-t'_1) \cdots (1-t'_{k-1}) \mathrm{cost} (\beta_k,S^{t'_k}_{j,k}) \\
\geq & \sum_{\mathbf{t}' \in \{0,1\}^m } \delta_{U^1_j} (\mathbf{t}') \sum_{k=1}^m (1-t'_1) \cdots (1-t'_{k-1}) \min_{\gamma \in \mathcal{A}_{\mathrm{dir}} (T_{j,k}) }\mathrm{cost} (\gamma,S^{t'_k}_{j,k}) .
\end{aligned}
\end{equation}
Note that in these equations, $(1-t'_1) \cdots (1-t'_{j-1}) \neq 0$ if and only if $t'_1=\cdots= t'_{k-1}=0$, which holds when $\beta$ evaluates the root values of $T_{j,1},\ldots,T_{j,k-1}$ as all 0 and succeeds to evaluation of $T_{j,k}$. Otherwise, the computation halts before the evaluation of $T_{j,k}$.

By induction hypothesis of (3), without loss of generality, we can assume that for each $k \in \{1,\ldots,m\}$, $\beta_k$ minimizes the cost of both $S^0_{j,k}$ and $S^1_{j,k}$. This implies that
\begin{equation} \label{eq:betacostSprime0}
\begin{aligned}
 \mathrm{cost} (\beta,S^0_j) 
=& \mathrm{cost} (\beta,S'_0) \\
=& \sum_{k=1}^m \mathrm{cost} (\beta_k,S^0_{j,k}) \\
=& \sum_{k=1}^m \min_{\gamma \in \mathcal{A}_{\mathrm{dir}} (T_{j,k}) } \mathrm{cost} (\gamma,S^0_{j,k}) \\
\leq & \min_{\alpha_j \in \mathcal{A}_{\mathrm{dir}}(T_j)} \sum_{k=1}^m \mathrm{cost} ((\alpha_j)_k,S^0_{j,k}) \\
=& \min_{\alpha_j \in \mathcal{A}_{\mathrm{dir}}(T_j)} \mathrm{cost} (\alpha_j,S'_0) \\
=& \min_{\alpha_j \in \mathcal{A}_{\mathrm{dir}}(T_j)} \mathrm{cost} (\alpha_j,S^0_j),
\end{aligned}
\end{equation}
that is, $\beta$ also minimizes the cost of $S^0_j$. Hence, by definition of $S^1_j$ and \eqref{eq:Sprimewritten}, we have
\begin{equation} \label{eq:partialmincost}
\begin{aligned}
& \min_{\alpha_j \in \mathcal{A}_{\mathrm{dir}}(T_j)} \mathrm{cost} (\alpha_j,S') \\
=& \min_{\alpha_j \in \mathcal{A}_{\mathrm{dir }}(T_j)}  [q_j \mathrm{cost} (\alpha_j,S^0_j) + (1-q_j) \mathrm{cost} (\alpha_j,S'_1)] \\
=& q_j \mathrm{cost} (\beta,S^0_j) + (1-q_j) \mathrm{cost} (\beta,S'_1) \\
=& q_j \min_{\alpha_j \in \mathcal{A}_{\mathrm{dir }}(T_j)}  \mathrm{cost} (\alpha_j,S^0_j) + (1-q_j) \min_{\alpha_j \in \mathcal{A}_{\mathrm{dir}(T)}} \mathrm{cost} (\alpha_j,S'_1) \\
\leq & q_j \min_{\alpha_j \in \mathcal{A}_{\mathrm{dir }}(T_j)}  \mathrm{cost} (\alpha_j,S^0_j) + (1-q_j) \min_{\alpha_j \in \mathcal{A}_{\mathrm{dir}(T)}} \mathrm{cost} (\alpha_j,S^1_j) \\
\leq & \min_{\alpha_j \in \mathcal{A}_{\mathrm{dir }}(T_j)}  [q_j \mathrm{cost} (\alpha_j,S^0_j) + (1-q_j) \mathrm{cost} (\alpha_j,S^1_j)] \\
=& \min_{\alpha_j \in \mathcal{A}_{\mathrm{dir }}(T_j)}  \mathrm{cost} (\alpha_j,S^\delta_j).
\end{aligned}
\end{equation}
By \eqref{eq:subtreecostineq} and \eqref{eq:partialmincost}, we get \eqref{eq:partialcostineq}. (End of the proof of Claim 2.)

Considering that $p_j$, $q_j$ and $S_j$ depend only on $S$ and $\sigma$, Claim 2 implies that
\begin{equation} \label{eq:fixedordermincost}
\begin{aligned}
 \min_{\alpha: \text{ order } \sigma} \mathrm{cost} (\alpha, S) 
=& \min_{\alpha: \text{ order } \sigma} \left[ \sum_{j=1}^n p_j \mathrm{cost} (\alpha_j,S_j) \right] \\
=& \sum_{j=1}^n p_j \min_{\alpha_j \in \mathcal{A}_{\mathrm{dir }}(T_j)} \mathrm{cost} (\alpha_j,S_j) \quad 
\text{[by $\alpha$ directional]}\\
\leq & \sum_{j=1}^n p_j \min_{\alpha_j \in \mathcal{A}_{\mathrm{dir }}(T_j)}  \mathrm{cost} (\alpha_j,S^\delta_j) \\ 
=& \min_{\alpha: \text{ order } \sigma} \left[ \sum_{j=1}^n p_j \mathrm{cost} (\alpha_j,S^\delta_j) \right] \\
=& \min_{\alpha: \text{ order } \sigma} \mathrm{cost} (\alpha, S^\delta).
\end{aligned}
\end{equation}
(End of the proof of Claim 1.)

By Claim 1, we have
\begin{equation} \label{eq:globalmincostineq}
\begin{aligned}
\min_{\alpha \in \mathcal{A}_\mathrm{dir}(T)} \mathrm{cost} (\alpha, S)
=& \min_{\sigma\in\mathfrak{S}_n} \min_{\alpha: \text{ order } \sigma} \mathrm{cost} (\alpha, S) \\
\leq & \min_{\sigma\in\mathfrak{S}_n} \min_{\alpha: \text{ order } \sigma} \mathrm{cost} (\alpha, S^\delta) \\
=& \min_{\alpha \in \mathcal{A}_\mathrm{dir}(T)} \mathrm{cost} (\alpha, S^\delta),
\end{aligned}
\end{equation}
where $\sigma$ runs the permutations on $\{1,\ldots,n\}$. This completes the proof of (2).

(3) By (2) and $T$ having AND root, the distribution
\begin{equation} \label{eq:definS1}
S^1 := \prod_{j=1}^n S^1_j
\end{equation}
achieves $P_{\mathrm{dir},1}(T)$. Let $\tilde{S} \in \mathcal{D}(\Omega_0(T))$ be a distribution which achieves $P_{\mathrm{dir},0}(T)$ and
\begin{equation} \label{eq:definS0}
S^0 := \langle \delta_{\tilde{S}} ; \{ S^i_j \}_{j,i} \rangle = \sum_{\mathbf{t}\in\{0,1\}^n} \delta_{\tilde{S}} (\mathbf{t}) \prod_{j=1}^n S^{t_j}_j .
\end{equation}
By (2), $S^0$ also achieves $P_{\mathrm{dir},0}(T)$. Let $\tilde{\alpha} \in \mathcal{A}_{\mathrm{dir}} (T)$ be a minimizer of $S^0$, that is, a directional algorithm which satisfies
\begin{equation} \label{eq:S0minimizer}
\mathrm{cost} (\tilde{\alpha},S^0) = \min_{\alpha \in \mathcal{A}_{\mathrm{dir}} (T) } \mathrm{cost} (\alpha,S^0).
\end{equation}
Without loss of generality, we can assume that $\tilde{\alpha}$ evaluates the subtrees $T_1,\ldots,T_n$ in numerical order. Then we have
\begin{equation} \label{eq:S0mincost}
 \mathrm{cost} (\tilde{\alpha},S^0) 
= \sum_{\mathbf{t}\in\{0,1\}^n} \delta_{\tilde{S}} (\mathbf{t}) \sum_{j=1}^n t_1 \cdots t_{j-1} \mathrm{cost} (\tilde{\alpha}_j, S^{t_j}_j).
\end{equation}
Thus, by induction hypothesis of (3), we can assume that for each $j\in\{1,\ldots,n\}$, $\alpha_j$ minimizes the cost of both $S^0_j$ and $S^1_j$. This implies that
\begin{equation} \label{eq:S1mincost}
\begin{aligned}
 \mathrm{cost} (\tilde{\alpha},S^1)
=& \sum_{j=1}^n \mathrm{cost} (\tilde{\alpha}_j,S^1_j) \\
=& \sum_{j=1}^n \min_{\alpha_j \in \mathcal{A}_{\mathrm{dir}} (T_j) } \mathrm{cost} (\alpha_j,S^1_j) \\
\leq & \min_{\alpha \in \mathcal{A}_{\mathrm{dir}} (T) } \sum_{j=1}^n \mathrm{cost} (\alpha_j,S^1_j)\\
=& \min_{\alpha \in \mathcal{A}_{\mathrm{dir}} (T) } \mathrm{cost} (\alpha,S^1),
\end{aligned}
\end{equation}
in other words, $\tilde{\alpha}$ also minimizes the cost of $S^1$.
\end{proof}

\begin{lemm}
\label{lemm:sdelta_ndir0} 
Let $T$ be an AND--OR tree with nonzero height. Let $v_1,\ldots,v_n$ be the depth-1 nodes of $T$, and $T_1,\ldots,T_n$ the corresponding subtrees. For each $S \in \mathcal{D}(\mathcal{W}(T))$, set $\delta_S \in \mathcal{D}( \{ 0,1 \}^{n} )$ as
\begin{equation} \label{eq:prszerov}
\delta_S (t_{1} \cdots t_{n}) = \mathrm{Pr}_{S} [ \text{$\bar{v}$ is $\mathbf{t}$} ] .
\end{equation}
Then there exists a series of distributions $\{S^i_j\}_{i\in\{0,1\},j\in\{1,\ldots,n\}}$ which satisfies the following.

(1) For each $i$ and $j$, we have $S^i_j \in \mathcal{D}(\Omega_i(T_j))$ and it achieves $\overline{P}_{\mathrm{dir},i}(T_j)$, that is,
\[
\min_{A \in \mathcal{D} ( \mathcal{A}_{\mathrm{dir}} (T_j) )} \mathrm{cost} (A,S^i_j) = \overline{P}_{\mathrm{dir},i} (T_j).
\]

(2) For any $S \in \mathcal{D}(\Omega (T))$, we have

\begin{equation} \label{eq:costzdeltabars0}
\min_{Y \in \mathcal{D} ( \mathcal{A}_\mathrm{dir}(T) )} \mathrm{cost} (Y, S)
\leq 
\min_{Y \in \mathcal{D} ( \mathcal{A}_\mathrm{dir}(T) )} \mathrm{cost} (Y, S ^{\delta}),
\end{equation}

where $S^{\delta}$ denotes $\langle \delta_S; \{ S_{j}^{i} \}_{j,i} \rangle$. 
\end{lemm}

\begin{proof}
Let a series of distributions $\{S^i_j\}_{i\in\{0,1\},j\in\{1,\ldots,n\}}$ be as in Lemma \ref{lemm:sdelta_det}. Then, by Proposition \ref{prop:vonneumanneq}, each $S^i_j$ also achieves $\overline{P}_{\mathrm{dir},i} (T_j)$, which meets (1). Again by the lemma and proposition, for any $S \in \mathcal{D} ( \Omega (T) )$, we get
\begin{equation} \label{eq:dettorand}
\begin{aligned}
 \min_{Y \in \mathcal{D} ( \mathcal{A}_{\mathrm{dir}} (T) )} \mathrm{cost} (Y,S)
=& \min_{\alpha \in \mathcal{A}_{\mathrm{dir}} (T)} \mathrm{cost} (\alpha,S) \\
\leq & \min_{\alpha \in \mathcal{A}_{\mathrm{dir}} (T)} \mathrm{cost} (\alpha,S^\delta) \\
=& \min_{Y \in \mathcal{D} ( \mathcal{A}_{\mathrm{dir}} (T) )} \mathrm{cost} (Y,S^\delta),
\end{aligned}
\end{equation}
where $S^\delta = \langle \delta_S; \{ S_{j}^{i} \}_{j,i} \rangle$. This implies (2).
\end{proof}

\subsection{Construction of the ``chimera'' algorithm} \label{subsection:main}

Now we are going to define the ``chimera'' algorithm $B_{\alpha,\delta} \in \mathcal{D} (\mathcal{A}_\mathrm{DF}(T))$, which is the most important tool of the main result. The ``chimera" nickname of this algorithm arises from the fact that it is composed of parts extracted from various algorithms and probability distributions.

For each $j \in \{ 1,\dots, n\}$, suppose that $Y_{j} \in \mathcal{D}(\mathcal{A}_\mathrm{dir} (T_{j}) )$ and $S_{j}^{i} \in \mathcal{D}( \Omega_{i} (T_{j}) )$. 
Suppose that $\alpha \in \mathcal{A}_\mathrm{DF} (T)$. Note that $\alpha$ is not necessarily a directional algorithm. 
Let $\delta \in \{ 0, 1 \}^{n}$, and let $S^{\delta} = \langle \delta; \{ S_{j}^{i} \}_{j,i} \rangle$. 
When we define $B_{\alpha,\delta}$, we use not only $\alpha$ and $\delta$ but also $\{ Y_{j} \}_{j}$ and $\{ S_{j}^{i} \}_{j,i}$. However we omit writing $\{ Y_{j} \}_{j}$ and $\{ S_{j}^{i} \}_{j,i}$ in the suffix of $B_{\alpha,\delta}$.  

Recall that each element of $\mathcal{A}(T)$ is a Boolean decision tree such that each node is a leaf of $T$, each internal node of this Boolean decision tree has one incoming arrow and two outgoing arrows, and one outgoing arrow is labeled with Boolean value 0, and the other is labeled with Boolean value 1. 
We may regard an element of $\mathcal{D}(\mathcal{A} (T))$ as a probabilistic Boolean decision tree whose root is the empty string and each internal node may have plural outgoing arrows labeled with the same Boolean value; besides a Boolean value, an outgoing arrow is labeled with (conditional) probability so that the sum of probabilities of outgoing arrows from the same internal node with the same Boolean value is 1. 

\begin{defi} \label{defi:balphadelta}
We define $B_{\alpha,\delta}$ as follows.
\begin{itemize}
\item Let $j_{1}$ be the $j \in \{ 1,\dots, n \}$ such that $\alpha$ begins its computation with probing $T_{j}$. 
$B_{\alpha, \delta}$ begins with probing $T_{j_{1}}$ with probability 1.
\item For each $j \in \{ 1,\dots, n \}$, whenever $B_{\alpha, \delta}$ probes $T_{j}$, $B_{\alpha, \delta}$ works as $Y_{j}$. 
The fact that the definition of $Y_{j}$ is independent of $i \in \{ 0,1 \}$ (= a constraint on the value of $v_{j}$) works here. 
\item For each $k \in \{ 1,\dots, n \}$, we let ``$T_{j_{1}}, \dots, T_{j_{k}}$'' denote the following event: ``From the beginning of computation, $T_{j_{1}}, \dots, T_{j_{k}}$ are probed in this order (the computation may or may not halt with $T_{j_{k}}$.)'' 

\item For each $k \in \{ 1,\dots, n-1 \}$ and for each injection $\sigma : \{ 1,\dots, k+1 \} \to \{ 1, \dots, n \}$,
we define $B_{\alpha, \delta}$'s transition probability as follows. 
Suppose that $m$ is an internal node of $B_{\alpha,\delta}$ as a probabilistic Boolean decision tree (see the paragraph just before Definition) such that  just after $T_{\sigma (1)}, \dots, T_{\sigma (k)}$ happened $B_{\alpha,\delta}$  founds that $v_{\sigma (k)}$ is 1 at $m$ (provided that the root of $T$ is labeled with AND; ``0 at $m$'' for OR). 
For each outgoing arrow from $m$ to a leaf $\ell$ of $T_{\sigma (k+1)}$, which we denote by arrow $(m,\ell)$, we define its (conditional) probability as follows. 
\begin{align}
& \text{(The probability with which arrow $(m,\ell)$ is labeled)}
\notag
\\
:=&
\mathrm{Pr}_{\alpha, S^{\delta}} [T_{\sigma (1)}, \dots, T_{\sigma (k)}, T_{\sigma (k+1)} | T_{\sigma (1)}, \dots, T_{\sigma (k)}]
\notag
\\
& ~ \times \mathrm{Pr}_{Y_{\sigma (k+1)}} [\text{$Y_{\sigma (k+1)}$ begins with $\ell$} ]
\label{eq:transprob_b}
\end{align}
\end{itemize}
This completes the definition of $B_{\alpha,\delta}$.
\end{defi}

\begin{exam} \label{exam:evaluates1}
Suppose that $k=3 < n$, $i \in \{ 0,1 \}$ and $t_{j} \in \{ 0,1 \}$ ($j=1,2$).
\begin{enumerate}
\item 
Suppose that $\omega_{j} \in \Omega_{t_{j}} (T_{j})$ ($j=1,2$). 
Then, the following quantity is determined only by $k(=3)$, $i$, $t_{1}$, and $t_{2}$ without depending on the choice of $(\omega_{1}, \omega_{2})$. Let $p(k,i,t_{1},t_{2})$ be the common value. 

\begin{equation} \label{eq:eval1}
\mathrm{Pr}_{S^{\delta}} [v_{3} \text{ is } i |  \omega_{1}, \omega_{2} \text{ are assigned to } T_{1}, T_{2} \text{ respectively.} ]
\end{equation}
\item Suppose that $E$ is an event on $\Omega_{t_{1}} (T_{1}) \times \Omega_{t_{2}} (T_{2})$, that is, an event implying that the roots of $T_{1}$ and $T_{2}$ have value $t_{1}, t_{2}$, respectively. Then we have the following. In particular, $\mathrm{Pr}_{S^{\delta}} [v_{3} \text{ is } i | E ]$ is determined only by $k$, $i$, $t_{1}$, and $t_{2}$ without depending on the choice of $E$. 

\begin{equation} \label{eq:eval2}
\mathrm{Pr}_{S^{\delta}} [v_{3} \text{ is } i | E ]
=
p(3,i,t_{1},t_{2})
\end{equation}
\item 
Assume that the root of $T$ is labeled with AND. Then we have the following. 
\begin{equation} \label{eq:eval3}
\mathrm{Pr}_{\alpha, S^{\delta}} [v_{3} \text{ is } i | T_{1},T_{2},T_{3} ]
=
p(3,i,1,1)
\end{equation} 
\end{enumerate}
\end{exam}

\begin{proof}
(1) In the following, we denote event ``$\omega_{1}, \omega_{2}$ are assigned to $T_{1}, T_{2}$ respectively'' by ``$\omega_{1}, \omega_{2}$'', and denote event ``$\omega_{3}$ is  assigned to $T_{3}$ '' by ``$\omega_{3}$''. Similar convention applies to ``$\omega_{1}, \omega_{2}, \omega_{3}$'' as well. 

\begin{align} \label{eq:evalpr1}
\mathrm{Pr}_{S^{\delta}} [v_{3} \text{ is } i | \omega_{1}, \omega_{2} ]
&=\sum_{ \xi \in \Omega_{i} (T_{3})} \mathrm{Pr}_{S^{\delta}} [ \xi | \omega_{1}, \omega_{2} ]
\notag
\\
&=
\sum_{\xi \in \Omega_{i} (T_{3})} 
\frac{ \mathrm{Pr}_{S^{\delta}} [ \omega_{1}, \omega_{2}, \xi ] }{ \mathrm{Pr}_{S^{\delta}} [ \omega_{1}, \omega_{2} ] }
\notag
\\
&=
\sum_{\xi \in \Omega_{i} (T_{3})} 
\frac{ \sum_{\omega} S^{\delta} ( \omega_{1} \omega_{2} \xi \omega_{4} \cdots \omega_{n} ) }{ \sum_{\eta \in \Omega (T_{3})} \sum_{\omega} S^{\delta}  ( \omega_{1} \omega_{2} \eta \omega_{4} \cdots \omega_{n} ) }
\notag
\\
& \text{[$\omega =(\omega_{4},\dots,\omega_{n})$ runs over $\Omega (T_{4}) \times \cdots \times \Omega (T_{n})$]}
\end{align}

Let $t_{j} \in \{ 0,1 \}$ be the value of $v_{j}$ determined by $\omega_{j}$ for $j=1,2$. 
For each $\eta$, let $i_{\eta} \in \{ 0,1 \}$ be the value of $v_{3}$ determined by $\eta$. 
For each $\omega$, let $\mathbf{t}^{\omega }=t_{4}\cdots t_{n} \in \{ 0,1 \}^{n-3}$ be the values of $v_{4}, \dots, v_{n}$ determined by $\omega$. 
Let $P=S^{t_{1}}_{1}(\omega_{1}) \cdot S^{t_{2}}_{2}(\omega_{2})$. 
By Lemma \ref{lemm:prsdelta_deltau}, we have the following. 

\begin{align} \label{eq:evalpr2}
\mathrm{Pr}_{S^{\delta}} [v_{3} \text{ is } i | \omega_{1}, \omega_{2} ]
&=
\frac
{\sum_{\xi  \in \Omega_{i} (T_{3})} \sum_{\omega} \delta ( t_{1} t_{2} i \mathbf{t}^{\omega}) P S^{i}_{3} (\xi) \prod_{j=4}^{n} S^{t_{j}}_{j} (t_{j})}
{\sum_{\eta  \in \Omega (T_{3})} \sum_{\omega} \delta ( t_{1} t_{2} i_{\eta} \mathbf{t}^{\omega}) P S^{i_{\eta}}_{3} (\eta)  \prod_{j=4}^{n} S^{t_{j}}_{j} (t_{j})}
\notag
\\
&=
\frac
{\sum_{\xi  \in \Omega_{i} (T_{3})} S^{i}_{3} (\xi) \sum_{\omega} \delta ( t_{1} t_{2} i \mathbf{t}^{\omega}) \prod_{j=4}^{n} S^{t_{j}}_{j} (t_{j})}
{\sum_{\eta  \in \Omega (T_{3})} S^{i_{\eta}}_{3} (\eta)  \sum_{\omega} \delta ( t_{1} t_{2} i_{\eta} \mathbf{t}^{\omega}) \prod_{j=4}^{n} S^{t_{j}}_{j} (t_{j})}
\end{align}

Now, the value of the last formula is determined only by $k(=3)$, $i$, $t_{1}$, and $t_{2}$ without depending on the choice of $(\omega_{1}, \omega_{2}) \in \Omega (T_{1}) \times \Omega (T_{2})$. 
Thus we have shown (1). 

Assertion (2) follows immediately from assertion (1). 

(3) 
Let $E$ be the event ``$T_{1},T_{2},T_{3}$'' with respect to $\alpha$. 
Then $E$ is an event on $\Omega_{1} (T_{1}) \times \Omega_{1} (T_{2})$. 
We have $\mathrm{Pr}_{\alpha, S^{\delta}} [v_{3} \text{ is } i | T_{1},T_{2},T_{3} ]= \mathrm{Pr}_{S^{\delta}} [v_{3} \text{ is } i | E]$. Since we assumed that the root of $T$ is labeled with AND, $E$ implies that the roots of $T_{1}$ and $T_{2}$ have value 1. Therefore we can apply (2) to $\mathrm{Pr}_{S^{\delta}} [v_{3} \text{ is } i | E]$, which equals $p(3,i,1,1)$. 

\end{proof}

\begin{lemm} \label{lemm:balphadeltaclaim}
\begin{enumerate}
\item For each $\alpha \in \mathcal{A}_\mathrm{DF}(T)$, we have $B_{\alpha, \delta} \in \mathcal{D} (\mathcal{A}_\mathrm{dir}(T))$ even if $\alpha$ is not directional.

\item We have $\mathrm{Pr} [\text{$B_{\alpha, \delta}$ evaluates $T_{j}$}] = \mathrm{Pr}_{S^{\delta}} [\text{$\alpha$ evaluates $T_{j}$}] $, where $\mathrm{Pr}$ without suffix denotes probability measured by $B_{\alpha,\delta}$ and $S^{\delta}$. 

\item For each $j \in \{ 1, \dots, n \}$ and $i \in \{ 0,1 \}$, wee have the following, where $\mathrm{Pr}$ without suffix is as above. 
\begin{equation}
\mathrm{Pr} [\text{$v_{j}$ has value $i$}|\text{$B_{\alpha, \delta}$ evaluates $T_{j}$}]
= \mathrm{Pr}_{S^{\delta}} [\text{$v_{j}$ has value $i$}|\text{$\alpha$ evaluates $T_{j}$}]
\end{equation} 
\end{enumerate}
\end{lemm}

\begin{proof}
(1) is immediate from the definition of $B_{\alpha,\delta}$. 

(2) Let $k \in \{ 1,\dots, n-1 \}$ and $\sigma : \{ 1,\dots, k+1 \} \to \{ 1, \dots, n \}$ be an injection. We let $L_{j}$ denote the leaves of $T_{j}$. 
In the context where we look at the end of computation of $T_{\sigma (k)}$, 
we let ``last $m$'' denote the event that the last leaf (of $T_{\sigma (k)}$) probed is $m$, 
and we let ``next $\ell$'' denote the event that (after $T_{\sigma (k)}$) the first leaf probed is $\ell$. 
We have the following, 
where $\mathrm{Pr}$ without suffix denotes probability measured by $B_{\alpha,\delta}$ and $S^{\delta}$. 

\begin{align}
& 
\mathrm{Pr} [T_{\sigma (1)}, \dots, T_{\sigma (k)}, T_{\sigma (k+1)} | T_{\sigma (1)}, \dots, T_{\sigma (k)}]
\notag
\\
=&
\sum_{\ell \in L_{\sigma (k+1)}} 
\mathrm{Pr} [ \text{ next } \ell  | T_{\sigma (1)}, \dots, T_{\sigma (k)}]
\notag
\\
=&
\sum_{\substack{m \in L_{\sigma (k)} \\ \ell \in L_{\sigma (k+1)}}} 
\mathrm{Pr} [ \text{ last } m \land \text{ next } \ell | T_{\sigma (1)}, \dots, T_{\sigma (k)}]
\notag
\\
=&
\sum_{\substack{m \in L_{\sigma (k)} \\ \ell \in L_{\sigma (k+1)}}} 
\mathrm{Pr} [ \text{ next } \ell | T_{\sigma (1)}, \dots, T_{\sigma (k)}  \land \text{ last } m ]
\times 
\mathrm{Pr}[ \text{ last } m | T_{\sigma (1)}, \dots, T_{\sigma (k)} ]
\notag
\\
=&
\sum_{\substack{m \in L_{\sigma (k)} \\ \ell \in L_{\sigma (k+1)}}} 
\mathrm{Pr}_{\alpha, S^{\delta}} [  T_{\sigma (1)}, \dots, T_{\sigma (k+1)} | T_{\sigma (1)}, \dots, T_{\sigma (k)}  ]
\notag
\\
& \quad 
\times 
\mathrm{Pr}_{Y_{\sigma (k+1)}} [Y_{\sigma (k+1)} \text{ begins with } \ell]
\times 
\mathrm{Pr}[ \text{ last } m | T_{\sigma (1)}, \dots, T_{\sigma (k)} ]
\notag
\\
& \quad \text{[by \eqref{eq:transprob_b}]}
\notag
\\
=&
\mathrm{Pr}_{\alpha, S^{\delta}} [T_{\sigma (1)}, \dots, T_{\sigma (k)}, T_{\sigma (k+1)} | T_{\sigma (1)}, \dots, T_{\sigma (k)}]
\label{eq:transprob_b2}
\end{align} 

By \eqref{eq:transprob_b2} and induction on $k^{\prime} \in \{ 1,\dots, n \}$,
\begin{equation}
\mathrm{Pr} [T_{\sigma (1)}, \dots, T_{\sigma (k^{\prime})}]
= \mathrm{Pr}_{\alpha, S^{\delta}} [T_{\sigma (1)}, \dots, T_{\sigma (k^{\prime})}]
\end{equation}
Thus, (2) holds. 

(3) It suffices to show the following claim.

\ 

{\bf Claim.} Let $j_1,\ldots,j_s \in \{1,\ldots,n\}$ be mutually distinct numbers and $j_s=j$. Then we have
\begin{equation}
\mathrm{Pr} [ v_{j_s} \text{ is } i | T_{j_1} , T_{j_2} , \ldots , T_{j_s} ] = \mathrm{Pr}_{\alpha,S^{\delta}} [ v_{j_s} \text{ is } i | T_{j_1} , T_{j_2} , \ldots , T_{j_s} ].
\end{equation}

For simplicity, we only see the case where $s=3$, $j_1=1$, $j_2=2$, and $j_3=3$. The general case can be shown similarly.

We have

\begin{align} \label{eq:ex4badpr4a}
&\mathrm{Pr}_{B_{\alpha,\delta}, S^{\delta}} [v_{3} \text{ is } i \land T_{1},T_{2},T_{3} ]
\notag
\\
=&
\sum_{\omega_{1}, \omega_{2}} \mathrm{Pr}_{B_{\alpha,\delta}, S^{\delta}} [v_{3} \text{ is } i \land \omega_{1}, \omega_{2} \land T_{1},T_{2},T_{3}  ]
\notag
\\
& \text{[$\omega_{j}$ runs over $\Omega_{1} (T_{j})$ for $j=1,2$]}
\notag
\\
=&
\sum_{\omega_{1}, \omega_{2}} \mathrm{Pr}_{S^{\delta}} [v_{3} \text{ is } i \land \omega_{1}, \omega_{2} ] 
\times 
\mathrm{Pr}_{B_{\alpha, \delta}} [ T_{1},T_{2},T_{3} | v_3 \text{ is } i \land \omega_{1}, \omega_{2} ]
\notag
\\
=&
\sum_{\omega_{1}, \omega_{2}} \mathrm{Pr}_{S^{\delta}} [v_{3} \text{ is } i \land \omega_{1}, \omega_{2} ] 
\times 
\mathrm{Pr}_{B_{\alpha, \delta}} [ T_{1},T_{2},T_{3} | \omega_{1}, \omega_{2} ]
\notag
\\
=&
\sum_{\omega_{1}, \omega_{2}} \mathrm{Pr}_{S^{\delta}} [v_{3} \text{ is } i | \omega_{1}, \omega_{2} ] 
\times
\mathrm{Pr}_{S^{\delta}} [\omega_{1},\omega_{2}] 
\times 
\mathrm{Pr}_{B_{\alpha, \delta}} [ T_{1},T_{2},T_{3} | \omega_{1}, \omega_{2} ]
\notag
\\
=&
\sum_{\omega_{1}, \omega_{2}} \mathrm{Pr}_{S^{\delta}} [v_{3} \text{ is } i | \omega_{1}, \omega_{2} ] 
\times
\mathrm{Pr}_{B_{\alpha, \delta},S^{\delta}} [ \omega_{1}, \omega_{2} \land T_{1},T_{2},T_{3}   ]
\end{align} 

By Example~\ref{exam:evaluates1}, we get the following. 
\begin{equation} \label{eq:ex4badpr4b}
\mathrm{Pr}_{S^{\delta}} [v_{3} \text{ is } i | \omega_{1}, \omega_{2} ] = \mathrm{Pr}_{\alpha,S^{\delta}} [v_{3} \text{ is } i | T_{1},T_{2},T_{3} ]
\end{equation}
Thus we have the following. 

\begin{align} \label{eq:ex4badpr4c}
&\mathrm{Pr}_{B_{\alpha,\delta}, S^{\delta}} [v_{3} \text{ is } i \land T_{1},T_{2},T_{3} ]
\notag
\\
=&
\sum_{\omega_{1}, \omega_{2}} \mathrm{Pr}_{\alpha,S^{\delta}} [v_{3} \text{ is } i | T_{1},T_{2},T_{3} ] 
\times
\mathrm{Pr}_{B_{\alpha, \delta},S^{\delta}} [ \omega_{1}, \omega_{2} \land T_{1},T_{2},T_{3}   ]
\notag
\\
=&
\mathrm{Pr}_{\alpha,S^{\delta}} [v_{3} \text{ is } i | T_{1},T_{2},T_{3} ] 
\times
\mathrm{Pr}_{B_{\alpha, \delta},S^{\delta}} [ T_{1},T_{2},T_{3} ]
\end{align}

Dividing both sides by $\mathrm{Pr}_{B_{\alpha, \delta},S^{\delta}} [ T_{1},T_{2},T_{3} ]$, 
we get 
$\mathrm{Pr}_{B_{\alpha,\delta},S^{\delta}} [v_{3} \text{ is } i | T_{1},T_{2},T_{3} ] = \mathrm{Pr}_{\alpha,S^{\delta}} [v_{3} \text{ is } i | T_{1},T_{2},T_{3} ]$. 

\end{proof}

\subsection{Proof of main theorem}

\begin{lemm}
\label{lemm:1distsubtreemin}
Let $T$ be an AND--OR tree and $T_1,\ldots,T_n$ its depth-1 subtrees. Assume that the root of $T$ is an AND node. For each $j\in\{1,\ldots,n\}$, let $S_j\in\mathcal{D}(\Omega_1(T_j))$, that is, $S_j$ be a distribution which gives the value 1 to the root of $T_j$, and for each $\omega=\omega_1\cdots\omega_n\in\Omega(T)$, let
\begin{equation} \label{eq:1distsubtreemin01}
S(\omega) = \prod_{j=1}^n S_j(\omega_j).
\end{equation}
Then, for any $X\in\mathcal{D}(\mathcal{A}_{\mathrm{DF}}(T))$, we have
\begin{equation} \label{eq:1distsubtreemin02}
\mathrm{cost} (X,S) \geq \sum_{j=1}^n \min_{Y_j\in\mathcal{D}(\mathcal{A}_{\mathrm{DF}}(T_j))} \mathrm{cost} (Y_j,S^1_j).
\end{equation}
If the root of $T$ is an OR node, then similar inequality holds with $S_j \in\mathcal{D}(\Omega_0(T_j))$ for each $\in\{1,\ldots,n\}$.
\end{lemm}

To show Lemma \ref{lemm:1distsubtreemin}, it is sufficient to show \eqref{eq:1distsubtreemin02} with arbitrary $\alpha\in\mathcal{A}_{\mathrm{DF}}(T)$ in the place of $X$. For a detailed proof, see Appendix.

\begin{theo} \label{theo:rdfvsrdirforp} 
For each $i=0,1$, it holds that $\overline{P}_{\mathrm{DF},i}(T)=\overline{P}_{\mathrm{dir},i}(T)$
\end{theo}

\begin{proof}
We begin with showing assertion by induction on the height of $T$. 

There exists $A \in  \mathcal{D}(\mathcal{A}_\mathrm{dir}(T))$ such that 
for each $i \in \{ 0, 1 \}$, there exists 
$S^{i} \in \mathcal{D}(\Omega_{i}(T))$ with the following properties. 
\begin{align}
\mathrm{cost} (A,S^{i})
&=
\overline{P}_{\mathrm{DF},i}(T)
\notag
\\
=&
\min_{X \in  \mathcal{D}(\mathcal{A}_\mathrm{DF}(T))} \mathrm{cost} (X,S^{i})
\notag
\\
=&
\min_{Y \in  \mathcal{D}(\mathcal{A}_\mathrm{dir}(T))} \mathrm{cost} (Y,S^{i})
\notag
\\
=&
\overline{P}_{\mathrm{dir},i}(T)
\end{align}

The case when height is 0 is obvious. 
Suppose that the assertion of the theorem holds for all trees of less height. 
Suppose that the root of $T$ is labeled with AND. 

By induction hypothesis, for each $j \in \{ 1,\dots, n\}$, we can fix $Y_{j} \in \mathcal{D}(\mathcal{A}_\mathrm{dir} (T_{j}) )$ and $S_{j}^{i} \in \mathcal{D}( \Omega_{i} (T_{j}) )$ with the following properties. 
\begin{align}
\mathrm{cost} (Y_{j},S_{j}^{i})
&=
\overline{P}_{\mathrm{DF},i}(T_{j})
\notag
\\
=&
\min_{X \in \mathcal{D}(\mathcal{A}_\mathrm{DF}(T_{j}))} \mathrm{cost} (X,S^{i}_j)
\notag
\\
=&
\min_{Y \in  \mathcal{D}(\mathcal{A}_\mathrm{dir}(T_{j}))} \mathrm{cost} (Y,S^{i}_j)
\notag
\\
=&
\overline{P}_{\mathrm{dir},i}(T_{j}) \label{eq:costgammasjk}
\end{align}

We begin with investigating $\overline{P}_{\mathrm{DF},0} (T) $. 
Fix $S \in \mathcal{D}(\Omega_{0} (T))$ with the following property. 
\begin{equation} \label{eq:propertyofs}
\min_{Y \in \mathcal{D}(\mathcal{A}_\mathrm{dir} (T) )} \mathrm{cost}(Y, S) 
= \overline{P}_{\mathrm{dir},0} (T) 
\end{equation}

We define a set $\Omega^{\text{depth-1}}_{0}$ of strings as follows. Note that each element of this set is not a truth assignment to the leaves but that to depth-1 nodes. 
\begin{equation}
\Omega^{\text{depth-1}}_{0} :=\{ \mathbf{t} \in \{ 0,1 \}^{n} : 
\text{$\bar{v}$ is $\mathbf{t}$ } 
\implies \text{the root of $T$ has value $0$.}
\}
\end{equation}

Let $\delta \in \mathcal{D}( \Omega^{\text{depth-1}}_{0} )$ be the following probability distribution. 
\(
\delta (t_{1} \cdots t_{n}) = \mathrm{Pr}_{S} [ \text{$\bar{v}$ is $\mathbf{t}$ } ] 
\). 
Let $S^{\delta}$ denote $\langle \delta; \{ S_{j}^{i} \}_{j,i} \rangle$ $($see Definition~\ref{defi:sdeltadir}$)$. 
For each $j \in \{ 1,\dots, n \}$, $i \in \{ 0,1 \}$, $\alpha \in \mathcal{A}_\mathrm{DF}(T)$, and $\beta \in \mathcal{A}_\mathrm{DF}(T_{j})$, we define the following events.

$E_{j} (\alpha): $ ~ $\alpha$ probes $T_{j}$. 

$E_{j}^{(i)}: $ ~ $v_{j}$ has value $i$. 

$E_{j}^{\prime} (\alpha, \beta): $ ~ $\alpha$ probes $T_{j}$ working as $\beta$.

Then, we have the following, where $\mathrm{Pr}$ is the probability measure given by $S^{\delta}$. 

\begin{align}
&\mathrm{cost}(\alpha, S^{\delta})
\notag
\\
= & 
\sum_{j=1}^{n} \mathrm{Pr}[E_{j} (\alpha)]
\Bigl(
\mathrm{Pr}[E_{j}^{(0)}  |  E_{j} (\alpha)]
\sum_{\beta \in \mathcal{A}_\mathrm{DF} (T_{j})} \mathrm{Pr}[E_{j}^{\prime} (\alpha, \beta)  |  E_{j} (\alpha) \land E_{j}^{(0)}]
\mathrm{cost}(\beta, S_{j}^{0})
\notag
\\
& +
\mathrm{Pr}[E_{j}^{(1)}  |  E_{j} (\alpha)]
\sum_{\beta \in \mathcal{A}_\mathrm{DF} (T_{j})} \mathrm{Pr}[E_{j}^{\prime} (\alpha, \beta)  |  E_{j} (\alpha) \land E_{j}^{(1)}]
\mathrm{cost}(\beta, S_{j}^{1})
\Bigr)
\end{align}

In each part of the form $\mathrm{cost}(\beta, S_{j}^{i})$, by Proposition~\ref{prop:vonneumanneq} (3), 
the following holds. 
\begin{align} \label{eq:costbetasjk}
\mathrm{cost}(\beta, S_{j}^{i}) 
\geq &\min_{\gamma \in \mathcal{A}_\mathrm{DF} (T_{j}) }\mathrm{cost}(\gamma, S_{j}^{i})
\notag
\\
= & \min_{\gamma \in \mathcal{A}_\mathrm{dir} (T_{j}) }\mathrm{cost}(\gamma, S_{j}^{i})
\notag
\\
=&
\mathrm{cost}(Y_{j}, S_{j}^{i})
\end{align}

To sum up, the following holds. 
\begin{align}
&\mathrm{cost}(\alpha, S^{\delta})
\notag
\\
\geq & 
\sum_{j=1}^{n} \mathrm{Pr}[E_{j} (\alpha)]
\Bigl(
\mathrm{Pr}[E_{j}^{(0)} | E_{j} (\alpha)]
\sum_{\substack{\beta \in \\ \mathcal{A}_\mathrm{DF} (T_{j}) }} \mathrm{Pr}[E_{j}^{\prime} (\alpha, \beta) | E_{j} (\alpha) \land E_{j,0}]
\mathrm{cost}(Y_{j}, S_{j}^{0})
\notag
\\
& +
\mathrm{Pr}[E_{j}^{(1)}  |  E_{j} (\alpha)]
\sum_{\beta \in \mathcal{A}_\mathrm{DF} (T_{j})} \mathrm{Pr}[E_{j}^{\prime} (\alpha, \beta)  |  E_{j} (\alpha) \land E_{j,1}]
\mathrm{cost}(Y_{j}, S_{j}^{1})
\Bigr)
\notag
\\
= & 
\sum_{j=1}^{n} \mathrm{Pr}[E_{j} (\alpha)]
\Bigl(
\mathrm{Pr}[E_{j}^{(0)}  |  E_{j} (\alpha)]
\mathrm{cost}(Y_{j}, S_{j}^{0})
+
\mathrm{Pr}[E_{j}^{(1)}  |  E_{j} (\alpha)]
\mathrm{cost}(Y_{j}, S_{j}^{1})
\Bigr) \label{eq:lboundalphasdelta}
\end{align}

Now we use $B_{\alpha,\delta}$ in Definition~\ref{defi:balphadelta} with respect to $\alpha$, $\delta$, $\{ Y_{j} \}_{j}$, and $\{ S_{j}^{i} \}_{j,i}$ in the current proof. 
By Lemma~\ref{lemm:balphadeltaclaim} and \eqref{eq:lboundalphasdelta}, we have the following, where $\mathrm{Pr}$ denotes probability measured by $B_{\alpha,\delta}$ and $S^{\delta}$. 

\begin{align}
&\mathrm{cost}(\alpha, S^{\delta})
\notag
\\
\geq & 
\sum_{j=1}^{n} \mathrm{Pr} [\text{$B_{\alpha,\delta}$ probes $T_{j}$}] 
\Bigl(
\mathrm{Pr} [\text{$v_{j}$ has value $0$}|\text{$B_{\alpha,\delta}$ probes $T_{j}$}]
\mathrm{cost}(Y_{j}, S_{j}^{0})
\notag
\\
& \qquad +
\mathrm{Pr} [\text{$v_{j}$ has value $1$}|\text{$B_{\alpha,\delta}$ probes $T_{j}$}]
\mathrm{cost}(Y_{j}, S_{j}^{1})
\Bigr) 
\notag
\\
= & \mathrm{cost}(B_{\alpha,\delta},S^{\delta})
\notag
\\ 
\geq &
\min_{\beta \in \mathcal{A}_\mathrm{DF} (T)} \mathrm{cost}(B_{\beta,\delta}, S^{\delta})
\label{eq:lboundalphasdelta2}
\end{align}

Since $\alpha \in \mathcal{A}_\mathrm{DF}(T)$ was arbitrary, any $X \in \mathcal{D}(\mathcal{A}_\mathrm{DF}(T))$ satisfies the following. 
\begin{align}
\mathrm{cost}(X, S^{\delta})
= & 
\sum_{\alpha \in \mathcal{A}_\mathrm{DF}(T)} 
X(\alpha) \mathrm{cost}(\alpha,S^{\delta})
\notag
\\
\geq &
\min_{\beta \in \mathcal{A}_\mathrm{DF} (T)} \mathrm{cost}(B_{\beta,\delta}, S^{\delta})
\label{eq:lboundssdelta}
\end{align}

Since $X \in \mathcal{D}(\mathcal{A}_\mathrm{DF}(T))$ was arbitrary, we have the following. 

\begin{align}
&\overline{P}_{\mathrm{DF},0} (T) 
\notag
\\
\geq & 
\min_{X \in \mathcal{D}(\mathcal{A}_\mathrm{DF} (T) )} \mathrm{cost}(X, S^{\delta}) 
\quad \text{[by definition of $\overline{P}_{\mathrm{DF},0}$]}
\notag
\\ 
\geq &
\min_{\beta \in \mathcal{A}_\mathrm{DF} (T)} \mathrm{cost}(B_{\beta,\delta}, S^{\delta})
\quad \text{[by \eqref{eq:lboundssdelta}]}
\notag
\\ 
\geq &\min_{Y \in \mathcal{D}(\mathcal{A}_\mathrm{dir} (T) )} \mathrm{cost}(Y, S^{\delta}) 
\quad \text{[by Claim (i), $B_{\beta,\delta}$ is directional.]}
\notag
\\ 
\geq &
\min_{Y \in \mathcal{D}(\mathcal{A}_\mathrm{dir} (T) )} \mathrm{cost}(Y, S) 
\quad \text{[by Lemma~\ref{lemm:sdelta_ndir0}]}
\notag
\\
= &
\overline{P}_{\mathrm{dir},0} (T) \quad \text{[by \eqref{eq:propertyofs}]}
\notag
\\
\geq & \overline{P}_{\mathrm{DF},0} (T)
\label{eq:pbardf0sandwitch}
\end{align}

Hence, we have the following.

\begin{align}
\overline{P}_{\mathrm{DF},0} (T) 
= & 
\min_{X \in \mathcal{D}(\mathcal{A}_\mathrm{DF} (T) )} \mathrm{cost}(X, S^{\delta}) 
\notag
\\ 
= & 
\min_{\beta \in \mathcal{A}_\mathrm{DF} (T)} \mathrm{cost}(B_{\beta,\delta}, S^{\delta})
\notag
\\
= &\min_{Y \in \mathcal{D}(\mathcal{A}_\mathrm{dir} (T) )} \mathrm{cost}(Y, S^{\delta}) 
\notag
\\ 
= &
\overline{P}_{\mathrm{dir},0} (T) 
\label{eq:thecommonvalue}
\end{align}

The common value of \eqref{eq:thecommonvalue} is $\mathrm{cost}(B_{\beta,\delta}, S^{\delta})$ for some $\beta$, and we have $A:=B_{\beta,\delta} \in \mathcal{D}(\mathcal{A}_\mathrm{dir} (T))$ even if $\beta$ is not directional. 

Next, we investigate $\overline{P}_{\mathrm{DF},1}$. 
Let $\langle S_{j}^{1} \rangle_{j=1}^{n}$ denote $S^{\delta^{\prime}}$ for the $\delta^{\prime}$ such that $\delta^{\prime} (1^{n}) = 1$. In other words: 
\begin{equation}
\langle S_{j}^{1} \rangle_{j=1}^{n} (\omega_{1} \cdots \omega_{n}) = \prod_{j=1}^{n} S_{j}^{1} (\omega_{j})
\end{equation}
For $A$ defined above, we have the following. 

\begin{align}
\mathrm{cost}(A, \langle S_{j}^{1} \rangle_{j=1}^{n}) 
&=
\sum_{j=1}^{n} \mathrm{cost}(Y_{j},S_{j}^{1}) 
\notag
\\
=&
\sum_{j=1}^{n} \overline{P}_{\mathrm{dir},1}(T_{j}) ~ (=\sum_{j=1}^{n} \overline{R}_{\mathrm{dir},1}(T_{j}))
\notag
\\
&\quad \text{[by Proposition~\ref{prop:vonneumanneq} and \eqref{eq:costgammasjk}]}
\notag
\\
=&\overline{R}_{\mathrm{dir},1}(T) ~ (=\overline{P}_{\mathrm{dir},1}(T) )
\notag
\\
&
\text{[by Proposition~\ref{prop:vonneumanneq} and Theorem \ref{theo:rdiri_is_di}(3)]}
\label{eq:costas11s1n_01}
\end{align}

We also have the following, where $\langle Y_{1} \to \dots \to Y_{n}\rangle$ denotes the algorithm executes $\{ Y_{j} \}_{j}$ in this order. 

\begin{align}
\min_{\substack{X \in \\ \mathcal{D}(\mathcal{A}_\mathrm{DF}(T))}} \mathrm{cost}(X, \langle S_{j}^{1} \rangle_{j=1}^{n}) 
\leq & 
\min_{\substack{Y \in \\ \mathcal{D}(\mathcal{A}_\mathrm{dir}(T))}} \mathrm{cost}(Y, \langle S_{j}^{1} \rangle_{j=1}^{n}) 
\notag
\\
\leq &\mathrm{cost}(\langle Y_{1} \to \dots \to Y_{n}\rangle, \langle S_{j}^{1} \rangle_{j=1}^{n}) 
\notag
\\
=&\sum_{j=1}^{n} \mathrm{cost}(Y_{j},S_{j}^{1}) 
\notag
\\
=&\sum_{j=1}^{n} \min_{\substack{X_{(j)} \in \\ \mathcal{D}(\mathcal{A}_\mathrm{DF}(T_{j}))}}\mathrm{cost}(X_{(j)},S_{j}^{1}) 
\quad \text{[by \eqref{eq:costgammasjk}]}
\notag
\\
\leq & \min_{\substack{X \in \\ \mathcal{D}(\mathcal{A}_\mathrm{DF}(T))}} \mathrm{cost}(X, \langle S_{j}^{1} \rangle_{j=1}^{n}) 
\label{eq:costas11s1n_02}
\end{align}

The last inequality of \eqref{eq:costas11s1n_02} follows from Lemma \ref{lemm:1distsubtreemin}. By \eqref{eq:costas11s1n_01} and \eqref{eq:costas11s1n_02}, we get

\begin{align}
\mathrm{cost}(A, \langle S_{j}^{1} \rangle_{j=1}^{n}) 
= & 
\min_{X \in \mathcal{D}(\mathcal{A}_\mathrm{DF} (T) )} \mathrm{cost}(X, \langle S_{j}^{1} \rangle_{j=1}^{n}) 
\notag
\\ 
= &\min_{Y \in \mathcal{D}(\mathcal{A}_\mathrm{dir} (T) )} \mathrm{cost}(Y, \langle S_{j}^{1} \rangle_{j=1}^{n}) 
\notag
\\ 
= &
\overline{P}_{\mathrm{dir},1} (T).
\label{eq:thecommonvalue2}
\end{align}

Moreover, we have the following. 

\begin{equation}
\min_{\substack{X \in \\ \mathcal{D}(\mathcal{A}_\mathrm{DF}(T))}} \mathrm{cost}(X, \langle S_{j}^{1} \rangle_{j=1}^{n}) 
\leq
\overline{P}_{\mathrm{DF},1}(T) 
\leq 
\overline{P}_{\mathrm{dir},1}(T)
\label{eq:costas11s1n_03}
\end{equation}

Hence, $\overline{P}_{\mathrm{DF},1}(T) $ equals the common value of \eqref{eq:thecommonvalue2}. This completes our proof of Theorem~\ref{theo:rdfvsrdirforp}. 
\end{proof}

\section{Concluding remarks} \label{sec:remark}

The following corollary sums up our result.

\begin{coro}
\label{coro:main}
Let $T$ be an AND--OR tree.
\begin{enumerate}
\item $R_{\mathrm{DF},i}(T) = R_{\mathrm{dir},i}(T) = d_i(T)$ for each $i\in\{0,1\}$.
\item $R_{\mathrm{DF}}(T) = R_{\mathrm{dir}}(T) = d(T)$.
\end{enumerate}
\end{coro}

\begin{proof}
The first equation directly follows from Theorems \ref{theo:rdiri_is_di} and \ref{theo:rdfvsrdirforp}. Let $A$ be a randomized algorithm which achieves $R_{\mathrm{DF}}(T)$. For each $i\in\{0,1\}$, we have
\begin{equation}
\label{eq:dfileqdf}
R_{\mathrm{DF},i}(T) \leq \max_{\omega\in\Omega_i} \mathrm{cost} (A,\omega) \leq \max_{\omega\in\Omega} \mathrm{cost} (A,\omega) = R_{\mathrm{DF}}(T).
\end{equation}
By Lemma \ref{lemm:sw86lem3_1}, there exists an RDA $B$ which achieves both $d_0(T)$ and $d_1(T)$. By the first equation, $B$ also achieves $R_{\mathrm{DF},0}(T)$ and $R_{\mathrm{DF},1}(T)$. Thus, by \eqref{eq:dfileqdf}, we get
\begin{equation}
R_{\mathrm{DF}}(T) \leq \max_{\omega\in\Omega} \mathrm{cost} (B,\omega) = \max \{ R_{\mathrm{DF},0}(T), R_{\mathrm{DF},1}(T) \} \leq R_{\mathrm{DF}}(T),
\end{equation}
which implies 
\begin{equation}
R_{\mathrm{DF}}(T) = \max \{ R_{\mathrm{DF},0}(T), R_{\mathrm{DF},1}(T) \}.
\end{equation}
Similarly, we have
\begin{equation}
R_{\mathrm{dir}}(T) = \max \{ R_{\mathrm{dir},0}(T), R_{\mathrm{dir},1}(T) \}
\end{equation}
and
\begin{equation}
d(T) = \max \{ d_0(T), d_1(T) \}.
\end{equation}
Hence we obtain the second equation.
\end{proof}

Together with Example \ref{exam:vereshchagin}, we also have the following.

\begin{coro}
There exists an AND--OR tree $V$ which satisfies $R(V) < R_{\mathrm{DF}}(V)$, that is, no randomized depth-first algorithm achieves $R(V)$.
\end{coro}

In \cite{It25}, it is shown that if an AND--OR tree $T$ satisfies $R(T)=d(T)$, there are infinitely many randomized algorithms which achieve $R(T)$. By our result, the assumption $R=d$ is equivalent to $R=R_{\mathrm{DF}}$.

\begin{coro}
If an AND--OR tree $T$ satisfies $R(T)=R_{\mathrm{DF}}(T)$, it has  uncountably many randomized algorithms which achieves $R(T)$.
\end{coro}

\section{Appendix} \label{sec:appendix}

In this section, we give detailed proofs we have omitted in the previous sections.

\subsection{Proof of Lemma \ref{lemm:sw86lem3_1}}

We perform induction on the height of the tree. 
If the height is 0 then the assertion of the lemma is obvious. 
Suppose that the height of $T$ is $h+1$ and assume that the assertion of the lemma holds for all trees of height at most $h$. 
Without loss of generality, we may assume that the root of $T$ is an AND node. 
Let $v_{1},\dots,v_{n}$ be the nodes just under the root of $T$. 
For each $k \in \{ 1,\dots,n \}$, let $T_{k}$  be the subtree whose root is $v_{k}$. 
For each $k$, by induction hypothesis, there is an algorithm $B_{k}$ of $T_{k}$ that achieves both $d_{0}(T_{k})$ and $d_{1}(T_{k})$. Fix such $B_{k}$ for each $k$.

\noindent
Claim 1 ~ If $X$ is an RDA on $T$ such that $X_{k}=B_{k}$ for each $k$ then $X$ achieves $d_{1} (T)$ regardless of the order of evaluation of $v_{1},\dots,v_{n}$ by $X$.

Proof of Claim 1: It is straightforward. Q.E.D.(Claim 1)

\noindent
Claim 2 ~ Suppose that $X$ is an RDA on $T$ that achieves $d_{0} (T)$, $k \in \{ 1, \dots, n \}$ and $Y$ is an RDA obtained by substituting $B_{k}$ for  $X_{k}$ then $Y$ achieves $d_{0} (T)$, too. 

Proof of Claim 2: 
It is enough to show the case of $k=1$. 
Let $\omega=\omega_{1}\cdots \omega_{n} $ be an assignment that maximizes $\mathrm{cost}(Y,\omega)$ under the constraint that the root has value 0, where each $\omega_{k}$ is the $T_{k}$-part of $\omega$. Let $i\in \{0,1\}$ be the value of $v_{1}$ in the presence of $\omega$. Then we may assume the following.
\begin{equation} \label{eq:swlem31a}
\mathrm{cost}(B_{1}, \omega_{1}) = \max_{\xi_{(1)}\in\Omega_i(T_1)} \mathrm{cost}(B_{1},\xi_{(1)})
\end{equation}

Now, let $\eta_{(1)}$ be an assignment to $T_{1}$ that make the value of $v_{1}$ $i$ and satisfying the following.
\begin{equation} \label{eq:swlem31b}
\mathrm{cost}(X_{1}, \eta_{(1)}) = \max_{\xi_{(1)}\in\Omega_i(T_1)} \mathrm{cost}(X_{1},\xi_{(1)})
\end{equation}

By \eqref{eq:swlem31a}, \eqref{eq:swlem31b} and the definition of $B_{1}$, we have the following. 
\begin{equation} \label{eq:swlem31c}
\mathrm{cost}(B_{1}, \omega_{1}) \leq \mathrm{cost}(X_{1}, \eta_{(1)})
\end{equation}

Hence, letting $\omega^{\prime}$ be the assignment given by substituting $\eta_{(1)}$ for $\omega_{1}$, we have the following. 
\begin{align} 
\max_{\xi\in\Omega_0(T)} \mathrm{cost}(Y,\xi )
&=
\mathrm{cost}(Y, \omega) \quad \text{[by the definition of $\omega$]}
\notag \\
&\leq \mathrm{cost}(X, \omega^{\prime}) \quad \text{[by \eqref{eq:swlem31c}]}
\notag \\
&\leq \max_{\xi\in\Omega_0(T)} \mathrm{cost}(X, \xi) 
\notag \\
&=d_{0}(T) \quad \text{[by the assumption on $X$]} \label{eq:swlem31d}
\end{align}
Hence $Y$ achieves $d_{0} (T)$. Q.E.D.(Claim 2)

By repeatedly applying Claim 2 to $X$ in the claim for $k \in \{1,\dots,n\}$, we get an RDA $Z$ such that $Z$ achieves $d_{0}(T)$ and $Z_{k}=B_{k}$ for all $k \in \{1,\dots,n\}$. 
By Claim 1, $Z$ achieves $d_{1} (T)$, too. 

\subsection{Proof of Lemma \ref{lemm:prsdelta_deltau}}

(1) It is easy to see that $\prod_{j=1}^{n} S_{j}^{t_{j}} (\omega_{j}) = 0$ for each $\mathbf{t}$ that does not satisfy the condition mentioned above. 

(2) 
\begin{align}
\sum_{\omega} S^{\delta} (\omega) 
=& 
\sum_{\omega \in \Omega(T)}
\sum_{ \mathbf{t}\in \{ 0,1 \}^{n}} \delta (\mathbf{t}) \prod_{j=1}^{n} S_{j}^{t_{j}} (\omega_{j})
\notag
\\
=&
\sum_{\mathbf{t}\in \{ 0,1 \}^{n} } 
\delta (\mathbf{t}) 
\sum_{\substack{\omega_{1} \in \Omega(T_{1}) \\ \cdots \\ \omega_{n} \in \Omega(T_{n}) } }
\prod_{j=1}^{n} S_{j}^{t_{j}} (\omega_{j})
\notag
\\
=& \sum_{\mathbf{t} \in \{ 0,1 \}^{n} } 
\delta (\mathbf{t}) \sum_{\omega_{1} \in \Omega(T_{1})} S_{1}^{t_{1}} (\omega_{1}) 
\sum_{ \substack{ \omega_{2} \in \Omega(T_{2}) \\ \dots \\ \omega_{n} \in \Omega(T_{n}) } } 
\prod_{j=2}^{n} S_{j}^{t_{j}} (\omega_{j})
\notag
\\
=& \sum_{ \mathbf{t}\in \{ 0,1 \}^{n} } 
\delta (\mathbf{t}) \cdot 1 \cdot 
\sum_{ \substack{ \omega_{2} \in \Omega(T_{2}) \\ \dots \\ \omega_{n} \in \Omega(T_{n}) } } 
\prod_{j=2}^{n} S_{j}^{t_{j}} (\omega_{j})
\notag
\\
=& \cdots 
\notag
\\
=& \sum_{\mathbf{t} \in \{ 0,1 \}^{n} } 
\delta (\mathbf{t}) \cdot 1 
\notag
\\
=& 1
\end{align}

(3)
Fix $\mathbf{t}$ and let $V_{j}=\Omega_{t_{j}}(T_{j})$. 
For each $\omega \in V_{1} \times \cdots \times V_{n}$, $\mathbf{t}^{\omega}$ in the sense of (1) equals this fixed $\mathbf{t}$. 
\begin{align}
\mathrm{Pr}_{S^{\delta}} [ \text{$\bar{v}$ is $\mathbf{t}$}]
=& \sum_{\omega \in V_{1} \times \cdots \times V_{n}} S^{\delta} (\omega) 
\notag
\\
=& \sum_{\omega \in V_{1} \times \cdots \times V_{n}} \delta (\mathbf{t}) \prod_{j=1}^{n} S_{j}^{t_{j}} (\omega_{j})
\quad \text{[by (1)]}
\notag
\\
=& \delta (\mathbf{t}) \sum_{\omega \in V_{1} \times \cdots \times V_{n}} \prod_{j=1}^{n} S_{j}^{t_{j}} (\omega_{j}) 
\notag
\\
=& \delta (\mathbf{t}) \quad \text{[$\sum \prod =1$, in the same way as (2)]}
\end{align}

\subsection{Proof of the two equations}

$\star$ Proof of \eqref{eq:partialdistcomp}:
 \[
S_j = q_j U^0_j + (1-q_j) U^1_j, 
~ 
S^{\delta}_j = q_j S^0_j + (1-q_j) S^1_j
\]

The first equation is straightforward. 
We show the following for each $i=0,1$ and for each $\xi \in \Omega (T_{j})$.

\begin{equation} \label{eq:partialdistcomp001}
S^{i}_{j} (\xi) = \mathrm{Pr}_{S^{\delta}} [\xi \text{ is assigned to } T_j \ | \ \alpha \text{ evaluates } T_j \text{ and } v_j \text{ is } i]
\end{equation}
Then the second equation of \eqref{eq:partialdistcomp} is shown in the same way as the first equation. 

Suppose that $i \in \{ 0,1 \}$ and that $\xi \in \Omega (T_{j})$. 
In the case when $\xi$ does not set $v_{j}$ (the root of $T_{j}$) to $i$, it is easy to see that the both sides of \eqref{eq:partialdistcomp001} is 0. 
In the following, we assume that $\xi$ sets $v_{j}$ to $i$. 
Without loss of generality, we may assume that $j=n$. 

For each $\omega = (\omega_{1},\dots,\omega_{n-1}) \in \Omega (T_{1}) \times \cdots \times \Omega (T_{n-1})$, letting $t_{k}$ be the value of $v_{k}$ set by $\omega_{k}$, put $\mathbf{t}^{\omega}=t_{1}\cdots t_{n-1} i$. 
By Lemma \ref{lemm:prsdelta_deltau}, we have the following. 

\begin{equation} \label{eq:partialdistcomp002}
S^{\delta} (\omega_{1} \cdots \omega_{n-1} \xi) = \delta (\mathbf{t}^{\omega}) S^{i}_{n} (\xi) \prod_{k=1}^{n-1} S_{k}^{t_{k}} (\omega_{k})
\end{equation} 

Sum of the right-hand side over all $\omega  = (\omega_{1},\dots,\omega_{n-1}) \in \Omega (T_{1}) \times \cdots \times \Omega (T_{n-1})$ such that $\alpha$ evaluates $T_{n}$ is as follows. 
Here, $\mathbf{t}^{\prime} = t^{\prime}_{1} \cdots t^{\prime}_{n-1}$ runs over all elements of $\{ 0,1 \}^{n-1}$ such that if we set $(v_{1},\dots,v_{n-1})$ to $(t^{\prime}_{1}, \dots t^{\prime}_{n-1})$ then $\alpha$ evaluates $T_{n}$. 

\begin{align} \label{eq:partialdistcomp003}
\sum_{\substack{\omega \\ \text{$\alpha$ evaluates $T_{n}$}}} S^{\delta} (\omega_{1} \cdots \omega_{n-1} \xi)
=&
S^{i}_{n} (\xi) \sum_{\substack{\omega \\ \text{$\alpha$ evaluates $T_{n}$}}}
\delta (\mathbf{t}^{\omega}) \prod_{k=1}^{n-1} S_{k}^{t_{k}} (\omega_{k})
\notag
\\
=&
S^{i}_{n} (\xi) \sum_{\mathbf{t}^{\prime}}
\sum_{\substack{\omega \\ \mathbf{t}^{\omega}=\mathbf{t}^{\prime} i } }
\delta (\mathbf{t}^{\omega}) \prod_{k=1}^{n-1} S_{k}^{t_{k}} (\omega_{k})
\notag
\\
=&
S^{i}_{n} (\xi) \sum_{\mathbf{t}^{\prime}}
\delta (\mathbf{t}^{\prime} i) 
\sum_{\substack{\omega \\ \mathbf{t}^{\omega}=\mathbf{t}^{\prime} i } }
\prod_{k=1}^{n-1} S_{k}^{t_{k}} (\omega_{k})
\notag
\\
=&
S^{i}_{n} (\xi) \sum_{\mathbf{t}^{\prime}}
\delta (\mathbf{t}^{\prime} i) 
\notag
\\
=&
S^{i}_{n} (\xi) \sum_{\mathbf{t}^{\prime}}
\mathrm{Pr}_{S_{\delta}} [ \overline{v} \text{ is } \mathbf{t}^{\prime} i] 
\quad 
\text{[By Lemma \ref{lemm:prsdelta_deltau}]}
\notag
\\
=&
S^{i}_{n} (\xi) \mathrm{Pr}_{S_{\delta}} \text{[$\alpha$ evaluates $T_n$ and  $v_n$ is $i$]}
\end{align}

Therefore, the following holds. 

\begin{align} \label{eq:partialdistcomp004}
&\mathrm{Pr}_{S^{\delta}} [\xi \text{ is assigned to } T_n \land \alpha \text{ evaluates } T_n \land  v_n \text{ is } i]
\notag
\\
=&\mathrm{Pr}_{S^{\delta}} [\xi \text{ is assigned to } T_n \land \alpha \text{ evaluates } T_n]
\notag
\\
=&
\sum_{\omega_{1},\dots,\omega_{n-1}} S^{\delta} (\omega_{1} \cdots \omega_{n-1} \xi) 
\qquad \text{[Sum over all $\omega$ such that $\alpha$ evaluates $T_{j}$.]}
\notag
\\
=&
S^{i}_{n} (\xi) \mathrm{Pr}_{S_{\delta}} 
\text{[$\alpha$ evaluates $T_n$ and  $v_n$ is $i$]}
\end{align}

Hence, we get \eqref{eq:partialdistcomp001} in the case of $j=n$. 
The other cases are shown in the same way.

\

$\star$ Proof of \eqref{eq:Sprimewritten}:
\[
S' = q_j S^0_j + (1-q_j) S'_1,
\]
where $S'_1 (\xi) = \langle \delta_{U^1_j} ; \{ S^i_{j,k} \}_{k,i} \rangle$.

Suppose $\xi \in \Omega (T_{j})$. For the subtree $T_j$ and $\xi\in\Omega(T_j)$, we define $\mathbf{t}^{\xi}$, $t^{\xi}_{k}$, and ``$\overline{v_{j}}$ is $\mathbf{t}^{\xi}$'' similarly as we did for the original tree $T$. Unless otherwise specified, the symbol $\prod_{k=1}^{m}$ stands for $\prod_{k=1}^{m} S_{j,k}^{t^{\xi}_{k}} (\xi_{k}) $. 

\begin{align}
S^{\prime} (\xi) 
=& \langle \delta_{S_{j}} ; \{ S_{j,k}^{i} \}_{k,i} \rangle (\xi)
\notag
\\
=& \delta_{S_{j}} (\mathbf{t}^{\xi}) \prod_{k=1}^{m} 
\quad \text{[by Lemma \ref{lemm:prsdelta_deltau}]}
\notag
\\
=& \mathrm{Pr}_{S_{j}} [ \overline{v_{j}} \text{ is } \mathbf{t}^{\xi}] \prod_{k=1}^{m} 
\quad \text{[by Lemma \ref{lemm:prsdelta_deltau}]}
\end{align}

By the first equation of \eqref{eq:partialdistcomp},
 we have the following. 

\begin{align} \label{eq:Sprimewritten01}
S^{\prime} (\xi) 
=& 
( q_{j} \mathrm{Pr}_{U^{0}_{j}} [ \overline{v_{j}} \text{ is } \mathbf{t}^{\xi}] 
+
(1-q_{j}) \mathrm{Pr}_{U^{1}_{j}} [ \overline{v_{j}} \text{ is } \mathbf{t}^{\xi}]) \prod_{k=1}^{m} 
\notag
\\
=&
( q_{j} \delta_{U^{0}_{j}} (\mathbf{t}^{\xi}) + (1-q_{j}) \delta_{U^{1}_{j}} (\mathbf{t}^{\xi}) ) \prod_{k=1}^{m} 
\quad \text{[by Lemma \ref{lemm:prsdelta_deltau}]}
\end{align}

Note that in the beginning of the proof of Claim 2, we have assumed that the root of $T_j$ is an OR node. In the case when $\xi$ sets $v_{j}$ to 0, we have $\mathbf{t}^{\xi} = 0^{m}$, $t^{\xi}_{k} = 0$ for each $k$, and $\delta_{U^{0}_{j}} (0^{m}) = 1$. Moreover, $\delta_{U^{1}_{j}} (0^{m}) = 0$ and thus by Lemma \ref{lemm:prsdelta_deltau}, it holds that $S^{\prime}_{1} (\xi) =0$. Therefore by \eqref{eq:Sprimewritten01},  we get the following. 

\begin{align}
S^{\prime} (\xi) 
=& (q_{j} \cdot 1 + (1-q_{j}) 0) \prod_{k=1}^{m} S_{j,k}^{0} (\xi_{k}) 
\notag
\\
=& q_{j} \prod_{k=1}^{m} S_{j,k}^{0} (\xi_{k}) 
\notag
\\
=& q_{j} S^{0}_{j} (\xi) + (1-q_{j})0
\quad \text{[by \eqref{eq:subtree0dist}]}
\notag
\\
=& q_{j} S^{0}_{j} (\xi) + (1-q_{j}) S_{1}^{\prime} (\xi)
\end{align}

In the case when $\xi$ sets $v_{j}$ to 1, we have $\delta_{U^{0}_{j}} (\mathbf{t}^{\xi}) = 0$ and $S^{0}_{j} (\xi) = 0$, thus the following holds by \eqref{eq:Sprimewritten01}.

\begin{align} 
S^{\prime} (\xi) 
=&
( q_{j} \cdot 0 + (1-q_{j}) \delta_{U^{1}_{j}} (\mathbf{t}^{\xi}) ) \prod_{k=1}^{m} 
\notag
\\
=&  q_{j} \cdot 0 +  (1-q_{j}) S_{1}^{\prime} (\xi) 
\quad \text{[by definition of $S_{1}^{\prime}$ and Lemma \ref{lemm:prsdelta_deltau}]}
\notag
\\
=& q_{j} S^{0}_{j} (\xi) + (1-q_{j}) S_{1}^{\prime} (\xi)
\end{align}

Thus, we have shown \eqref{eq:Sprimewritten}.

\subsection{Proof of Lemma \ref{lemm:1distsubtreemin}}

We show by induction on the number of nodes in $T$. We only see the case where $T$ has the AND root (the other case is similar). First, we show the following.

\ 

{\bf Claim.} For any deterministic depth-first algorithm $\alpha$, we have
\begin{equation} \label{eq:1distsubtreemin03}
\mathrm{cost} (\alpha,S) \geq \sum_{j=1}^n \min_{Y_j\in\mathcal{D}(\mathcal{A}_{\mathrm{DF}}(T_j))} \mathrm{cost} (Y_j,S_j).
\end{equation}

\ 

Proof of Claim: 
Let $T_k$ be the first depth-1 subtree $\alpha$ evaluates, and $T'$ be the tree created by cutting off $T_k$ from $T$. For each $\xi\in\Omega_1(T_k)$, let $\alpha_k$ be the $T_k$ part of $\alpha$, and $\alpha^\xi$ the procedure of $\alpha$ after evaluating $T_k$ when the $T_k$ part of the assignment is $\xi$. Let $S'$ be a distribution of $T'$ defined by
\begin{equation} \label{eq:1distsubtreemin04}
S'(\omega') = \prod_{j \neq k} S_{j} (\omega'_j).
\end{equation}
Then, by induction hypothesis, we have
\begin{equation} \label{eq:1distsubtreemin05}
\begin{aligned}
& \mathrm{cost} (\alpha,S) \\
=&  \sum_{\omega\in\Omega(T)} S(\omega) \mathrm{cost} (\alpha,\omega) \\
=& \sum_{\omega_1\in\Omega_1(T_1)} \cdots \sum_{\omega_n\in\Omega_1(T_n)} S_1(\omega_1) \cdots S_n(\omega_n) \mathrm{cost} (\alpha,\omega) \\
=& \sum_{\xi\in\Omega_1(T_k)} \sum_{\omega'\in\Omega_1(T')} S_k(\xi)S'(\omega') [\mathrm{cost} (\alpha_k,\xi) + \mathrm{cost} (\alpha^\xi,\omega')] \\
=&  \mathrm{cost} (\alpha_k,S_k) + \sum_{\xi\in\Omega_1(T_k)} S_k(\xi) \mathrm{cost} (\alpha^\xi,S') \\
\geq & \min_{Y_k\in\mathcal{D}(\mathcal{A}_{\mathrm{DF}}(T_k))} \mathrm{cost} (Y_k,S_k) + \sum_{\xi\in\Omega_1(T_k)} S_k(\xi) \sum_{j \neq k} \min_{Y_j\in\mathcal{D}(\mathcal{A}_{\mathrm{DF}}(T_j))} \mathrm{cost} (Y_j,S_j) \\
=& \sum_{j=1}^n \min_{Y_j\in\mathcal{D}(\mathcal{A}_{\mathrm{DF}}(T_j))} \mathrm{cost} (Y_j,S_j).
\end{aligned}
\end{equation}
\hfill
Q.E.D. (Claim)

By this claim, we have
\begin{equation} \label{eq:1distsubtreemin06}
\begin{aligned}
\mathrm{cost} (X,S) =& \sum_{\alpha\in\mathcal{A}_{\mathrm{DF}}(T)} X(\alpha) \mathrm{cost} (\alpha,S) \\
\geq & \sum_{\alpha\in\mathcal{A}_{\mathrm{DF}}(T)} X(\alpha) \sum_{j=1}^n \min_{Y_j\in\mathcal{D}(\mathcal{A}_{\mathrm{DF}}(T_j))} \mathrm{cost} (Y_j,S_j) \\
=& \sum_{j=1}^n \min_{Y_j\in\mathcal{D}(\mathcal{A}_{\mathrm{DF}}(T_j))} \mathrm{cost} (Y_j,S^1_j).
\end{aligned}
\end{equation}



\end{document}